\documentclass[twoside,11pt]{article}
\usepackage{blindtext}

\usepackage[preprint]{jmlr2e}

\usepackage{lastpage}
\jmlrheading{23}{2022}{1-\pageref{LastPage}}{1/21; Revised 5/22}{9/22}{21-0000}{Author One and Author Two}

\ShortHeadings{Block-Sphere Vector Quantization}{Block-Sphere Vector Quantization} 
\firstpageno{1}

\usepackage{amsmath}
\usepackage{amssymb}
\usepackage{amsthm} %
\usepackage{abbreviations}
\usepackage{natbib} 
    \setcitestyle{authoryear,round,citesep={;},aysep={,},yysep={;}}
    \renewcommand{\cite}[1]{\citep{#1}} 
    
\usepackage{hyperref} 
    \hypersetup{ %
        pdftitle={},
        pdfsubject={},
        pdfkeywords={},
        pdfborder=0 0 0,
        pdfpagemode=UseNone,
        colorlinks=true,
        linkcolor=black,
        citecolor=mydarkblue, 
        filecolor=mydarkblue,
        urlcolor=mydarkblue,
    }

\usepackage{url}
\usepackage{mathtools}
\usepackage{dsfont}
\usepackage{tabularx}
\usepackage{makecell}
\usepackage{graphicx}
\usepackage{subcaption}
\usepackage{booktabs}       
\usepackage{subfiles}

\usepackage[capitalize,noabbrev]{cleveref}

\usepackage[utf8]{inputenc} 
\usepackage[T1]{fontenc}    
\usepackage{hyperref}       
\usepackage{url}            
\usepackage{booktabs}       
\usepackage{multirow}
\usepackage{amsfonts}       
\usepackage{nicefrac}       
\usepackage{microtype}      
\usepackage{xcolor}         

\usepackage{graphicx}
\usepackage{subcaption}

\usepackage{amsmath}
    
\usepackage{amssymb}
\usepackage{mathtools}
\usepackage{amsthm}
\usepackage{enumitem}   
\usepackage{algorithm}
\usepackage[noend]{algpseudocode}
\usepackage{abbreviations}
\usepackage{tabularx}
\usepackage{booktabs}   
\usepackage{makecell}
\usepackage{dsfont}
\usepackage{pifont}
\usepackage{caption}
\newcommand{\bq}{\textup{\texttt{BlockQuant}}}
\newcommand{\bqmse}{\textup{\texttt{BlockQuant}\textsubscript{MSE}}}
\newcommand{\bqmseapprox}{\textup{\texttt{BlockQuant}\textsubscript{MSE,approx}}}
\newcommand{\bqbsm}{\textup{\texttt{BlockQuant}\textsubscript{BSM}}}

\newcommand{\bqub}{\textup{\texttt{BlockQuant}\textsubscript{UB}}}
\newcommand{\bqubapprox}{\textup{\texttt{BlockQuant}\textsubscript{UB,approx}}}

\newcommand{\tq}{\textup{\texttt{TurboQuant}}}
\newcommand{\tqmse}{\textup{\texttt{TurboQuant}\textsubscript{MSE}}}
\newcommand{\tqprod}{\textup{\texttt{TurboQuant}\textsubscript{PROD}}}

\newcommand{\rbq}{\textup{\texttt{RabitQ}}}
\newcommand{\rbqbsm}{\textup{\texttt{RabitQ}\textsubscript{BSM}}}
\newcommand{\rbqub}{\textup{\texttt{RabitQ}\textsubscript{UB}}}

\newcommand{\ed}{\textup{\texttt{EDEN}}}
\newcommand{\edbsm}{\textup{\texttt{EDEN}\textsubscript{BSM}}}
\newcommand{\edub}{\textup{\texttt{EDEN}\textsubscript{UB}}}
\newcommand{\qjl}{\textup{\texttt{QJL}}}

\newcommand{\mse}{\mathcal{D}\textsubscript{MSE}}
\newcommand{\iprod}{\mathcal{D}\textsubscript{IP}}

\newcommand{\sphere}{\mathbb{S}^{d-1}}

\usepackage[toc,page,header]{appendix}
\usepackage{minitoc}

\title{Block-Sphere Vector Quantization}

\author{\name Heesang Ann\thanks{Equal contribution.} \email sang3798@snu.ac.kr \\
       \addr 
       Seoul National University \\
       \name Joongkyu Lee\footnotemark[1] \email jklee0717@snu.ac.kr \\
       \addr 
       Seoul National University \\
       \name Min-hwan Oh \email minoh@snu.ac.kr \\
       \addr 
       Seoul National University}

\editor{My editor}

\doparttoc 
\faketableofcontents 

\begin{document}

\maketitle

\begin{abstract}
Vector quantization is a fundamental primitive for scalable machine learning systems, enabling memory-efficient storage, fast retrieval, and compressed inference.
Recent rotation-based quantizers such as $\ed$, $\rbq$, and $\tq$ have introduced strong guarantees and empirical performance, but the surrounding comparisons have been difficult to interpret because they rely on different distortion criteria, probability regimes, and implementation assumptions.
As our first contribution, we provide a unified theoretical comparison of these methods and show that their relative advantages are criterion-dependent rather than absolute: $\ed$ and $\tq$ are favorable for MSE distortion, $\ed$ is also effective for expected inner-product distortion, and $\rbq$ provides strong high-probability control.
This comparison further clarifies that $\ed$ provides particularly strong guarantees for expected distortion measures.
As our second contribution, we introduce \textit{Block-Sphere Quantization} ($\bq$), a new rotation-based block quantization algorithm designed around the spherical geometry of randomly rotated vectors.
Unlike coordinate-wise quantizers, $\bq$ quantizes blocks on the sphere, preserving the geometry of rotated embeddings more faithfully.
We prove that this block-spherical design theoretically improves over the baselines considered in this paper for both reconstruction MSE and expected inner-product distortion.
Our experiments on real embedding datasets and long-context LLM inference tasks
show practical gains that are consistent with our theoretical improvements.
\end{abstract}



\section{Introduction}

Vector quantization addresses a central bottleneck in large-scale machine
learning systems: storing, transmitting, and comparing massive collections of
high-dimensional vectors. These vectors appear as embeddings for retrieval,
gradients in distributed and federated learning, and key--value cache states
in long-context LLM inference.
In these settings, quantization reduces storage cost and memory traffic while
enabling efficient similarity computation.
It has long been a core tool for billion-scale similarity search and is becoming
increasingly important in LLM inference, where KV-cache memory grows with
batch size and context length and can become a major bottleneck
\citep{johnson2019billion,liu2024kivi}.

Recent quantizers using random rotation have attracted significant attention, showing strong performance even in the low-bit regime. 
$\tq$ explicitly targets both reconstruction MSE and inner-product distortion, and reports strong performance in simple LLM and nearest-neighbor tasks~\citep{zandieh2025turboquant}. 
$\rbq$ and its extensions are also rotation-based quantizers, but use a different spherical approximation based on normalized rotated grids; they have recently been explored for both approximate nearest-neighbor search and LLM
quantization~\citep{gao2024rabitq,gao2025practical,yang2025raana,gao2026revisiting}. In parallel, $\ed$, which also uses random rotation and coordinate-wise marginal
distributions, has been revisited in relation to $\tq$~\citep{vargaftik2021drive, vargaftik2022eden,benbasat2026note}. 
A major strength of these methods
is that they come with theoretical guarantees. However, these guarantees are stated under different criteria: $\tq$ and $\ed$ are primarily analyzed through expected distortion measures, whereas $\rbq$ is analyzed through high-probability bit-complexity guarantees. 
This makes a direct theoretical comparison difficult. 
Moreover, some recently discussed variants, such as
inner-product variants of $\ed$ and reconstruction variants of $\rbq$, have mostly been compared empirically~\citep{ben2026note, gao2026revisiting}.

\begin{table}[t]\centering
    \caption{\small{Comparison of rotation-based quantizers. 
    The suffixes BSM and UB denote the best-scalar MSE variant and the unbiased
    inner-product variant, respectively; MSE and PROD follow the original
    \tq{} notation.
    The entries for
    $b=1,~2,~3,~4$ report high-dimensional approximate values. 
    The large-$b$ column reports the
    high-rate asymptotic upper bounds of the form $C\cdot 4^{-b}$. 
    All values are computed in this work except~$^\dagger$, reported in~\citet{zandieh2025turboquant}. 
    }}
\label{table:comparison_quantizers}
\begin{tabular}{lccccc}
\toprule
\multirow{2}{*}{Quantizer for reconstruction}
& \multicolumn{5}{c}{MSE distortion} \\
\cmidrule(lr){2-6}
& $b=1$ & $b=2$ & $b=3$ & $b=4$ & Large $b$ \\\midrule
$\edbsm$& $0.363$& $0.117$& ${0.0345}$& ${0.0095}$& ${2.721}\cdot 4^{-b}$ \\
$\rbqbsm$& ${0.363}$& ${0.119}$& ${0.0374}$& ${0.0115}$& -- \\
$\tqmse$& $0.36^\dagger$& $0.117^\dagger$& $0.03^\dagger$& $0.009^\dagger$& ${2.721}\cdot 4^{-b~\dagger}$ \\
$\bqbsm (p=2)$(ours)& ${0.363}$& ${0.108}$& ${0.0297}$& ${0.0078}$& ${2.015}\cdot 4^{-b}$ \\
$\bqbsm (p=3)$(ours)& $\mathbf{0.357}$& $\mathbf{0.101}$& $\mathbf{0.0271}$& $\mathbf{0.0071}$& $\mathbf{1.770}\cdot 4^{-b}$ \\
\specialrule{0.8pt}{1pt}{2pt}
\multirow{2}{*}{Quantizer for inner product}
& \multicolumn{5}{c}{Inner product distortion} \\
\cmidrule(lr){2-6}
& $b=1$ & $b=2$ & $b=3$ & $b=4$ & Large $b$ \\\midrule
$\edub$\rule{0pt}{3ex}&$\frac{0.571}{d-1}$& $\frac{0.133}{d-1}$& $\frac{0.0358}{d-1}$& $\frac{0.0096}{d-1}$& $\frac{2.721}{d-1}\cdot 4^{-b}$ \\
$\rbqub$\rule{0pt}{3ex}& $\frac{0.571}{d-1}$& $\frac{0.135}{d-1}$& $\frac{0.0389}{d-1}$& $\frac{0.0117}{d-1}$& -- \\
$\tqprod$\rule{0pt}{3ex}& $\frac{1.57}{d}^\dagger$& $\frac{0.56}{d}^\dagger$& $\frac{0.18}{d}^\dagger$& $\frac{0.047}{d}^\dagger$& $\frac{17.09}{d}\cdot 4^{-b~\dagger}$ \\
$\bqub ~(p=2)~$(ours)\rule{0pt}{3ex}& $\frac{0.571}{d-1}$& $\frac{0.120}{d-1}$& $\frac{0.0306}{d-1}$& $\frac{0.0078}{d-1}$& $\frac{2.015}{d-1}\cdot 4^{-b}$ \\
$\bqub ~(p=3)~$(ours)\rule{0pt}{3ex}& $\frac{\mathbf{0.553}}{d-1}$& $\frac{\mathbf{0.113}}{d-1}$& $\frac{\mathbf{0.0279}}{d-1}$& $\frac{\mathbf{0.071}}{d-1}$& $\frac{\mathbf{1.770}}{d-1}\cdot 4^{-b}$ 
\\\bottomrule
\end{tabular}

\end{table}

In this paper, we close this gap by placing $\ed$, $\rbq$, and $\tq$ under a
unified theoretical framework. We compare these methods under three criteria:
reconstruction MSE, expected inner-product distortion, and high-probability
bit complexity. This comparison reveals that no existing method dominates
across all criteria. Rather, each algorithm reflects a different design
principle: $\ed$ and $\tq$ are strong for reconstruction MSE, $\ed$ is effective for expected inner-product distortion, and $\rbq$ provides the strongest
high-probability control. This unified view separates the role of the quantization codebook from that of the dequantization rule.

Motivated by these findings, we propose Block-Sphere Quantization
($\bq$), a rotation-based block quantizer designed to improve expected
distortion. Whereas $\ed$ and $\tq$ use coordinate-wise marginal distributions
after random rotation, $\bq$ groups coordinates into small blocks and uses
the exact block marginal distribution induced by the unit sphere to construct
its centroids. This direction is related in spirit to product quantization~\citep{jegou2011product, ge2013optimized},
but differs in a crucial way: unlike product quantization methods, the codebook of $\bq$ is not learned from data, but is
derived from the known spherical distribution induced by random rotation.
As a result, $\bq$ achieves the best reconstruction MSE and expected inner-product error among the rotation-based methods. 
Our lower-bound analysis further shows that increasing the block size moves this framework toward the ideal spherical quantization limit, and our empirical validation supports the theoretical findings.    
Our main contributions are summarized as follows:
\begin{itemize}
    \item \textbf{A unified comparison of $\ed$, $\rbq$, and $\tq$.}
    We provide a unified theoretical comparison of $\ed$, $\rbq$, and $\tq$
    under three criteria: reconstruction MSE, expected inner-product distortion,
    and high-probability bit-complexity guarantees.
    Our analysis shows that the relative strengths of these methods depend on the
    criterion: $\ed$ yields strong guarantees for expectation-based distortion
    criteria, while $\rbq$ provides stronger high-probability guarantees.

    \item \textbf{Block-Sphere Quantization. }
    We propose Block-Sphere Quantization ($\bq$), a rotation-based block
    quantizer that exploits the spherical geometry of randomly rotated unit
    vectors. Unlike coordinate-wise quantizers, $\bq$ quantizes
    low-dimensional blocks using the exact block marginal distribution
    induced by the unit sphere. We prove that $\bq$ improve both the MSE distortion and the expected inner-product
    distortion over the baselines considered in this paper, respectively
    (see Table~\ref{table:comparison_quantizers}). Our analysis covers both the derivation of approximate constants for small bit-widths and the high-rate asymptotic regime.

    \item \textbf{A sharper lower bound.}
    We revisit the Shannon lower bound for the MSE distortion, and we obtain a corrected lower bound for our problem setting. We further show that the idealized version of
    $\bq$ with block size $p=d$ has an MSE upper bound that closely matches
    this lower bound, both in order and in the leading constant. This
    indicates that incorporating block-spherical structure is a principled
    route toward near-optimal MSE distortion.

    \item \textbf{Empirical validation.}
    We show that $\bq$ achieves lower reconstruction and inner-product distortion
    than $\ed$, $\rbq$, and $\tq$ on real embedding data, while maintaining
    comparable runtime via efficient approximate nearest-centroid search.
    We further demonstrate that $\bq$ yields practical gains for KV-cache
    quantization in long-context LLM inference, supporting our theoretical findings.
\end{itemize}

\section{Preliminaries}
\textbf{Notations. } 
We denote the unit sphere and the unit ball
in $\mathbb{R}^d$ by $\sphere$ and $\mathbb{B}^{d}$, respectively. We denote the beta 
function by $\mathrm{Beta}$ and the gamma and  digamma functions by $\Gamma$ and $\psi(t):=\frac{d}{dt}\log\Gamma(t)=\Gamma'(t)/\Gamma(t)$, respectively. 

\subsection{Problem Settings}
We consider a randomized quantizer $\Qcal$, with quantization map
$Q:\mathbb{R}^d \rightarrow \{0,1\}^{b\cdot d}$, which maps a
$d$-dimensional vector to a $b\cdot d$-bit string. We denote the
corresponding dequantization map by
$Q^{-1}: \{0,1\}^{b\cdot d} \rightarrow \mathbb{R}^d$. The randomness of
$Q$ may arise, for example, from random rotations or random projection
matrices. 
Since vector norms can be stored separately, we focus
on the quantization of unit vectors $\xb\in\sphere$. 

\paragraph{Expected distortion metrics.} For a randomized quantizer $\Qcal$, 
the worst-case MSE of the reconstructed vector is defined as 
$$
\begin{aligned}
    \mse(\Qcal)
    :=
    \max_{\xb \in \sphere}\mathbb{E}_{Q}
    \left[\left\|\xb-Q^{-1}(Q(\xb))\right\|_2^2 \right]~.
\end{aligned}
$$
While MSE distortion captures reconstruction quality, inner products are the
fundamental quantities in similarity search and retrieval. 
Thus, we also define the worst-case mean squared error of inner-product estimation as
$$
\begin{aligned}
    \iprod(\Qcal):=
    \max_{\xb  \in \sphere}\max_{\yb  \in \sphere}
    \mathbb{E}_{Q}\left[\left\{\langle \yb, Q^{-1}(Q(\xb))\rangle-\langle \yb,\xb\rangle\right\}^2\right]~.
\end{aligned}
$$
\paragraph{Rotation-based quantizer.} We focus on quantizers that employ a random rotation, including $\ed$, $\rbq$, and $\tq$. Let $\xb \in \sphere$ be a unit vector and let $R$ be a Haar-distributed orthogonal matrix. Then the rotated vector
$R\xb$ is uniformly distributed on $\sphere$. Rotation-based quantizers exploit this fact by first applying the random rotation and then encoding the
rotated vector using a fixed codebook $\Ccal$. More formally, let $P_{\mathrm{code}}$ denote the nearest-codeword map associated with $\Ccal$, and let $P_{\mathrm{decode}}$ denote the corresponding decoder. A rotation-based quantizer encodes $\xb$ through $P_{\mathrm{code}}(R\xb)$. The raw reconstruction in the original coordinate system is
$\bar{\xb}=R^\top P_{\mathrm{decode}} \bigl(P_{\mathrm{code}}(R\xb)\bigr)$
where $R^\top = R^{-1}$. The final dequantized output need not be exactly $\bar{\xb}$; depending on the target objective, it may apply an
additional rescaling or correction to $\bar{\xb}$. All $\ed$, $\rbq$, and $\tq$ fit into this rotation-based framework, with
different choices of the codebook and dequantization rule.

Our first goal is to compare the expected distortion metrics $\mse$ and $\iprod$ for $\ed$, $\rbq$, and $\tq$ within a common framework. For rotation-based quantizers, the expectation $\mathbb{E}_{Q}$ appearing in these metrics is taken with respect to the random rotation matrix $R$. Beyond these expectation-based criteria, we also compare
their high-probability guarantees for quantization error~\citep{gao2026revisiting}.
Motivated by this comparison, we then design efficient quantizers that
improve both $\mse$ and $\iprod$.

\subsection{Existing Quantizers for Comparison}
\paragraph{$\ed$~\citep{vargaftik2021drive, vargaftik2022eden}.}
For a unit input vector $\xb$, $\ed$ first applies a random rotation $R$ and rescales the rotated vector by a factor $\eta_{\mathrm{q}}$, so that a dimension-independent scalar codebook can be used. Each coordinate of the rescaled vector is then quantized to its nearest centroid. During dequantization, each code is replaced by the corresponding centroid, the inverse rotation is applied, and the result is rescaled by a scalar $\eta_{\mathrm{dq}}$. The choice of $\eta_{\mathrm{dq}}$ depends on the target objective.

A key component of $\ed$ is the Lloyd--Max scalar codebook $\mathcal{C}_{\mathrm{EDEN}}$, constructed for the standard normal distribution~\citep{vargaftik2022eden}. 
This choice is motivated by the distribution of randomly rotated vectors: if $\xb\in\sphere$ and $R$ is a random rotation, then each coordinate of $R\xb$ tends to follow $N(0,{1\over d})$. 
Thus, choosing $\eta_{\mathrm{q}}=\sqrt{d}$ makes the standard-normal
Lloyd--Max codebook well matched to the rotated coordinates.

The same quantization rule can be expressed equivalently
by scaling the codebook rather than the rotated vector $R\xb$. If we define $\mathcal{C}_{\mathrm{EDEN}}^{\mathrm{(scaled)}}=\frac{1}{\sqrt{d}}\mathcal{C}_{\mathrm{EDEN}}$, then the output of $\ed$ can be written as $\eta R^\top
\left(
P_{\mathrm{decode}}\left(P_{\mathrm{code}}(R\xb)\right)
\right)=\eta \bar{\xb}$ for some $\eta\in\RR$, where $P_{\mathrm{code}}$ and
$P_{\mathrm{decode}}$ denote the encoding and decoding maps with respect to
$\mathcal{C}_{\mathrm{EDEN}}^{\mathrm{(scaled)}}$.
Then, $\edbsm$  uses $\eta=\frac{\langle \xb,\bar{\xb}\rangle}{\|\bar{\xb}\|_2^2}$, which is the best scalar for minimizing the expected squared reconstruction error, and $\edub$ uses $\eta=\frac{1}{\langle \xb,\bar{\xb}\rangle}$ for unbiased inner-product estimation.  These two variants correspond to $\texttt{EDEN-biased}$ and $\texttt{EDEN-unbiased}$ in~\citet{ben2026note}, respectively.
For presentation clarity and consistency across methods, we rename them
$\edbsm$ and $\edub$, where the suffixes indicate the reconstruction-oriented
and unbiased inner-product variants, respectively.

\paragraph{$\rbq$~\citep{gao2025practical,gao2026revisiting}.}
$\rbq$ is a rotation-based quantizer whose practical encoder quantizes the rotated vector using a uniform grid. Although $\rbq$ is implemented through this grid quantization procedure, it can be equivalently viewed as using a fixed spherical codebook $\mathcal{C}_{\mathrm{RabitQ}}$ in the rotated domain~\citep{gao2025practical}. The elements of $\mathcal{C}_{\mathrm{RabitQ}}$ are obtained by projecting the grid codewords onto $\sphere$. Thus, after applying a random rotation $R$, $\rbq$ selects a spherical codeword for $R\xb$, and the raw reconstruction $\bar{\xb}$ is obtained by applying the inverse rotation to the decoded codeword.

$\rbq$ applies a scalar correction to the raw reconstruction for reconstruction or inner-product estimation, like $\ed$. For notational consistency with $\ed$ and $\tq$, we absorb the scalar correction into the dequantized output. Specifically, throughout this paper, we regard $\frac{\bar{\xb}}{\langle \xb,\bar{\xb}\rangle}$ as the dequantized output of $\rbqub$ for unbiased inner-product estimation. This convention is equivalent to the estimator in~\citet{gao2025practical}, where the correction factor is treated as part of the estimation procedure rather than as part of the dequantization output. For reconstruction, the best scalar multiple of $\bar{\xb}$ for approximating $\xb$ is $\frac{\langle \xb,\bar{\xb}\rangle}{\|\bar{\xb}\|_2^2}\bar{\xb} =\langle \xb,\bar{\xb}\rangle\bar{\xb}$, since $\rbq$ uses a spherical codebook. We therefore regard $\langle \xb,\bar{\xb}\rangle\bar{\xb}$ as the dequantized output of $\rbqbsm$. Here, $\rbqub$ corresponds to the original $\rbq$ in~\citet{gao2025practical}, while $\rbqbsm$ corresponds to $\rbq_{\mathrm{MSE}}$ in~\citet{gao2026revisiting}.
Following the same naming convention as above, we use the suffixes to distinguish
the unbiased inner-product variant from the reconstruction-oriented variant.

\paragraph{$\tq$~\citep{zandieh2025turboquant}.}
$\tq$ is a randomized rotation-based quantization method based on coordinate-wise scalar quantization, similar to $\ed$. Unlike $\ed$, which designs its scalar codebook using a Gaussian approximation to the marginal distribution of randomly rotated coordinates, $\tq$ constructs a Lloyd--Max codebook for the exact coordinate marginal of a uniformly random point on $\sphere$. In this sense, $\tqmse$ can be viewed as an EDEN-type coordinate-wise Lloyd--Max quantizer that replaces the high-dimensional Gaussian approximation with the exact spherical marginal and uses the scaling convention $\eta_{\mathrm{q}}=1$~\citep{ben2026note}.

$\tq$ also provides a variant for inner product estimation, denoted by $\tqprod$. A $b$-bit $\tqprod$ first applies a $(b-1)$-bit $\tqmse$ quantizer to obtain a raw reconstruction $\bar{\xb}$, and then allocates the remaining one bit to a $\qjl$-based correction for the residual $\xb-\bar{\xb}$~\citep{zandieh2025qjl}. The final estimator combines the inner product with the raw reconstruction and an unbiased one-bit estimate of the residual contribution. This residual correction makes the resulting inner-product estimator unbiased.

\section{New Theoretical Guarantees for Comparing Existing Quantizers}
\label{sec:theory}
In this section, we compare the theoretical guarantees of $\ed$, $\rbq$,
and $\tq$ from several perspectives. These algorithms report performance
under different criteria; the existing guarantees are summarized in
Appendix~\ref{app_sec:existing}. 
Specifically, \citet{vargaftik2022eden}
provides an $\mse$ bound for $\edub$, but does not report guarantees for
$\iprod$ or high-probability behavior. $\rbq$ proves a bit-complexity
guarantee based on high-probability analysis~\citep{gao2025practical}, but
does not provide expected-distortion guarantees such as $\mse$ or $\iprod$.
$\tq$ provides bounds for both $\mse$ and $\iprod$, but does not
analyze high-probability behavior. We fill these gaps and thereby enable a
unified comparison of existing methods under each criterion.

Overall, no single quantizer dominates across all measures, in particular, in high dimension, $\edbsm$ and $\tqmse$
show the better $\mse$ approximates than $\rbq$, whereas $\edub$ performs best for $\iprod$. Based
on this observation, in Section~\ref{sec:alg}, we propose quantizers that strictly improve upon existing methods under both criteria.

\subsection{MSE Comparison}

To compare reconstruction performance, we report either numerical approximations or upper bounds for the $\mse$ of $\edbsm$, $\rbqbsm$, and $\tqmse$. We begin with $\edbsm$: for small bit-widths $b\in\{1,2,3,4\}$, we report numerical approximations, whereas for the large-$b$ regime we use a high-rate upper bound.

\begin{proposition}[MSE of $\edbsm$]\label{prop:ed_mse}
    In high dimensions, the MSE of $\edbsm$ is approximately
    $\mse(\edbsm)\approx
    \mathbf{0.363},\;\mathbf{0.117},\;\mathbf{0.0345},\;\mathbf{0.0095}$ for $b=1,~2,~3,~4$, respectively. 
    Moreover, for large $b$, $\edbsm$ satisfies $\mse(\edbsm)
\le \mathbf{2.721}\cdot{1 \over {4^b}}$. 
\end{proposition}
\begin{proposition}[MSE of $\rbqbsm$]\label{prop:rbq_mse}
    Let $\zb=(z_1, \ldots,z_d)$ be randomly rotated vector of input unit vector and $R_j:=\sqrt{d}\,z_j$ be the rescaled coordinates of $\zb$ for $j=1,\ldots,d$. 
    Define the function $Q_b(u)= \operatorname{sgn}(u)
      \min\left(\lfloor |u|\rfloor+\frac12, 2^{b-1}-\frac12\right)$, where the value at $u=0$ is irrelevant for continuous distributions. Then the MSE of $\rbqbsm$ is 
\begin{equation*}
    \mse(\rbqbsm)= \mathbb{E}\left[
        \min_{\alpha>0}\frac1d\sum_{j=1}^d
        \left(R_j-\alpha Q_b(R_j/\alpha)\right)^2
      \right]. 
\end{equation*}
    Moreover, for large $d$, if we apply the gaussian approximations, the MSE bound is approximately $\mse(\rbqbsm) \approx
\mathbf{0.363},\;\mathbf{0.119},\;\mathbf{0.037},\;\mathbf{0.0115}~$ for $~b=1,~2,~3,~4$.
\end{proposition}
\begin{remark}[MSE of $\tqmse$]\label{rmk:tq_mse} 
    For $\tq$, \citet{zandieh2025turboquant} report approximate MSE values for $b=1,~2,~3,~4$. Since the reported numerical precision differs from the one used in our comparison, we recompute the constants under our notation : For large $d$, the MSE of $\tq$ is approximately $\mse (\tqmse) \approx \mathbf{0.363}, \;\mathbf{0.117}, \;\mathbf{0.0345}, \;\mathbf{0.0095}~$ for $~b=1,~2,~3,~4$, respectively.
\end{remark}
\paragraph{Discussion on MSE guarantees.} From Propositions~\ref{prop:ed_mse} and~\ref{prop:rbq_mse}, together with
Remark~\ref{rmk:tq_mse}, in high dimensions, both $\edbsm$ and $\tqmse$ achieve smaller approximate MSE values than $\rbq$ across $b=1,~2,~3,~4$.
The approximated MSE values of $\rbqbsm$ are also comparable, but become slightly larger than those of $\tqmse$ and $\edbsm$ as $b$ increases. Moreover, in the high-rate regime, $\edbsm$ admits an MSE upper bound with the same leading constant as that of $\tqmse$. Overall, these results suggest that the Lloyd--Max centroids used in $\tq$ and $\ed$, which are optimized for coordinate-wise MSE distortion, are effective in reducing $\mse$. The proofs of Proposition~\ref{prop:ed_mse}, Proposition~\ref{prop:rbq_mse} and explanation of Remark~\ref{rmk:tq_mse} are deferred to Appendix~\ref{subsec:ed_mse},~\ref{subsec:rbq_mse}, and~\ref{subsec:tq_mse}, respectively. 

\subsection{Inner Product Distortion Comparison}

In this section, we derive expected inner product estimation error for $\edub$ and $\rbqub$, then compare the bound with that of $\tqprod$. Different from $\tqprod$ which uses additional $\qjl$ algorithm for unbiased estimation, $\edub$ and $\rbqub$ apply scalar multiplication to raw reconstruction, i.e. multiplying ${1 \over \langle \xb, \bar{\xb}\rangle}$ to $\bar{\xb}$. In this paper, we call this type of quantizer by (rotation-based) ratio quantizer for inner product. The below theorem shows the general bound on inner product distortion of ratio quantizers. 

\begin{theorem}[General bound on rotation-based ratio quantizer]
\label{thm:ratio_est}
Let $\xb, \yb \in \sphere$ and $\eta:=\langle \xb, \yb \rangle$. Then, the inner product estimation from the ratio quantizers $\widehat\eta_{\rm ratio}
  :={\left\langle{\bar{\xb} \over{\langle{\bar{\xb}},{\xb}\rangle}},{\yb}\right\rangle}$, the squared error of $\widehat\eta_{\rm ratio}$ can be expressed by
\begin{equation*}
  \mathbb{E}[(\widehat\eta_{\rm ratio}-\eta)^2]
  =
  \frac{1-\eta^2}{d-1}
  \mathbb{E}\left[\frac{\|\bar{\xb}\|_2^2-\langle{\bar{\xb}},{\xb}\rangle^2}{\langle{\bar{\xb}},{\xb}\rangle^2}\right].
\end{equation*}
\end{theorem}
Then, we can get the inner product distortion guarantee of $\ed$ and $\rbq$. 
\begin{corollary}[Expected inner product error of $\edub$]
\label{cor:ed_ip}
In high dimension, the inner-product error of $\edub$ is approximately $\iprod(\edub)\le
    {\mathbf{0.571} \over d-1},\;{\mathbf{0.133} \over d-1},\;{\mathbf{0.0358} \over d-1},\;{\mathbf{0.0096} \over d-1}$, for $b=1,~2,~3,~4$, respectively. Moreover, for large $b$, $\edub$ satisfies
$$
\begin{aligned}
    \iprod(\edub)&\le\frac{\mathbf{2.721}}{d-1}4^{-b}(1+o(1)).
\end{aligned}
$$
\end{corollary}

\begin{corollary}[Expected inner product error of $\rbq$]
\label{cor:rbq_ip}
In high dimension, the inner product error of $\rbqub$ is approximately 
$\iprod(\rbqub) \approx
    {\mathbf{0.571} \over d-1},\;{\mathbf{0.135} \over d-1},\;{\mathbf{0.0389} \over d-1},\;{\mathbf{0.0117} \over d-1}~$ for $~b=1,~2,~3,~4$, respectively. 
\end{corollary}
\paragraph{Discussion on inner product error guarantees.}
    For inner-product estimation, $\edub$ gives the strongest guarantee among the methods compared here, closely followed by that of $\rbqub$. For $b=2,3,4$, its expected inner-product distortion is roughly four times smaller than that of $\tqprod$, whose distortion is approximately $\iprod(\tqprod) \approx
    {\mathbf{1.57} \over d},\;{\mathbf{0.56} \over d},\;{\mathbf{0.18} \over d},\;{\mathbf{0.047} \over d}~$respectively, as stated in Proposition~\ref{prop:tq_ip}. This comparison suggests that the rescaling step used in $\edub$ and $\rbqub$ are more effective for reducing expected inner product distortion than the $\qjl$-based correction used in $\tqprod$. The reason is that $\tqprod$ sacrifices one bit for the $\qjl$-based bias correction, rather than using all bits for reconstruction. This bit allocation appears to incur an approximately constant-factor loss, which is visible as the roughly fourfold gap. The proofs of Theorem~\ref{thm:ratio_est} and Corollaries~\ref{cor:ed_ip},~\ref{cor:rbq_ip} are deferred to Appendix~\ref{sec:IP}.
\subsection{High Probability Bit Complexity Comparison}

Unlike the other two methods, $\rbq$ analyzes the high-probability behavior of the inner-product error and derives the
corresponding bit-complexity guarantee. As pointed out by
\citet{gao2026revisiting}, applying Chebyshev's inequality only to the
variance bound for $\tqprod$ in Proposition~\ref{prop:tq_ip} is not
sufficient to recover the optimal bit complexity of \citet{alon2017optimal}
stated in Lemma~\ref{lem:alon}. We show, however, that a sharper
high-probability guarantee for $\tqprod$ can be obtained by combining the
high-probability guarantee of $\tqmse$ with that of $\qjl$. In particular, this refined analysis shows that $\tqprod$ attains the optimal bit complexity in a certain low-accuracy regime. The similar bit-complexity guarantee also holds for $\edub$. 

\begin{theorem}[Bit complexity of $\edub$ and $\tqprod$ (informal)]
\label{thm:ed_tq_bc}
Let $d$ be the input dimension and let $\epsilon,\delta\in(0,1)$. Suppose
that $\frac{1}{\epsilon}\log\left(\frac{1}{\delta}\right)\lesssim d $. Then, for both $\edub$ and $\tqprod$, it suffices to use $b=\Theta\left(
    \log\left(\frac{1}{d\epsilon^2}\log\frac{1}{\delta}\right)\right)$
bits per dimension to ensure that the inner-product estimation error is at
most $\epsilon$ with failure probability at most $\delta$.
\end{theorem}
\paragraph{Discussion on bit complexity.}
The above theorem shows that $\edub$ and $\tqprod$ achieves the optimal bit
complexity of \citet{alon2017optimal} when
$\frac{1}{\epsilon}\log\left(\frac{1}{\delta}\right)\lesssim d $. This guarantee is weaker than that of $\rbq$ in Lemma~\ref{prop:rbq_bc}, since $\rbq$ achieves the same optimal bit complexity in the high-accuracy regime $d \le \frac{1}{\epsilon^2}\log\left(\frac{1}{\delta}\right)$, which is the regime considered by \citet{alon2017optimal}. This limitation comes from the MSE-oriented construction of $\ed$ and $\tq$: although it effectively reduces expected reconstruction error, it does not directly provide sharp
high-probability control of the residual distribution. The proof is deferred to Appendix~\ref{sec:bc}.

\section{Proposed Method}\label{sec:alg}

Motivated by the comparison in Section~\ref{sec:theory}, we propose a new
quantization scheme that improves both $\mse$ and $\iprod$. Our method
follows the spirit of product quantization by grouping multiple coordinates
into blocks, and extends the coordinate-wise approaches used in $\ed$ and
$\tq$ to a block-spherical quantization scheme. This allows the quantizer to
incorporate the spherical structure of the input vectors more directly,
leading to improved distortion guarantees.

\subsection{Block-Sphere Quantization (\bq)} \label{subsec:bq}
\begin{figure}[t!]
    \centering
    \includegraphics[width=0.85\textwidth, trim=1cm 6cm 1cm 5.5cm, clip]{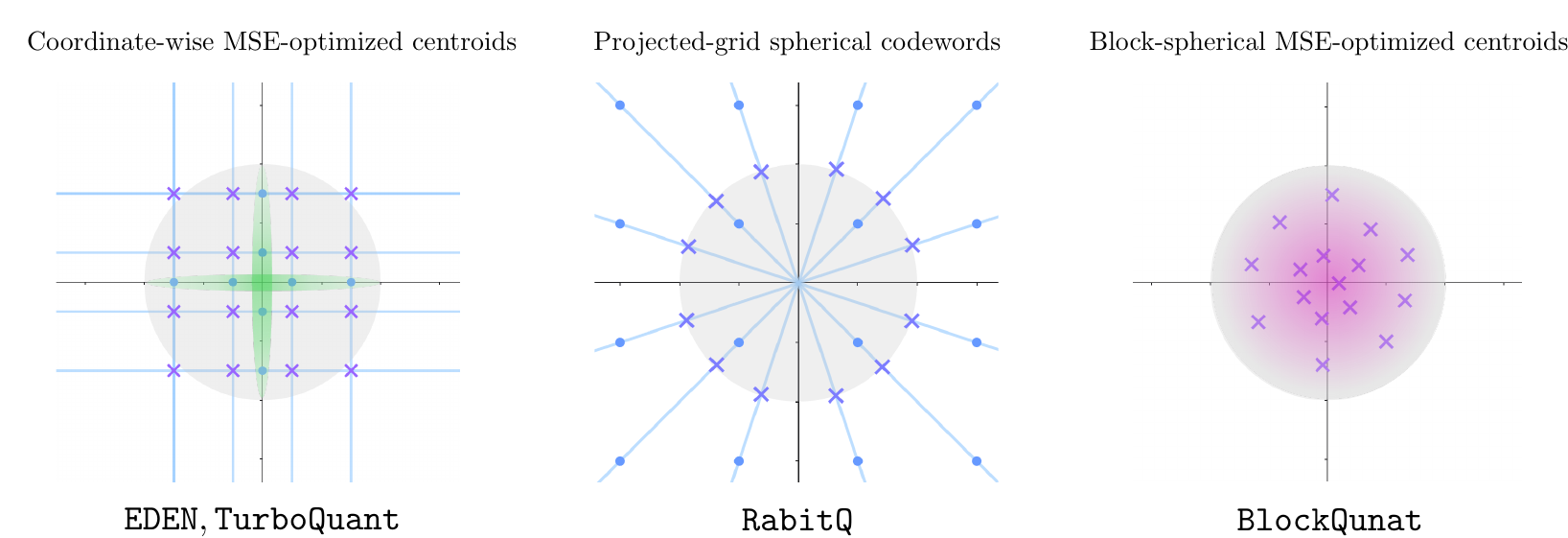}
    \vspace{5pt}
    \caption{\small{
Conceptual comparison of codebooks used by rotation-based quantizers for $b=2$ in a two-coordinate projection when $d>2$. The shaded disk indicates the feasible region of the two displayed coordinates of a rotated unit vector. Left: $\ed$ and $\tq$ use a Cartesian-product codebook formed by coordinate-wise MSE-optimized scalar centroids. Middle: $\rbq$ uses spherical codewords obtained by projecting a $2^b\times 2^b$ uniform grid onto $\sphere$. Right: $\bq$ with block size $p=2$ optimizes $2^{bp}=16$ centroids directly for the two-dimensional block distribution.}}
    \label{fig:codebook}
\end{figure}

The coordinate-wise MSE-optimized centroids used in $\ed$ and $\tq$ lead to low MSE distortion, but they do not fully exploit the spherical structure of unit input vectors. 
Since quantization is performed independently across coordinates, many possible combinations of scalar centroids can lie far from
the sphere. 
This suggests that part of the $2^{bd}$ code space is spent on codewords that are not well aligned with the geometry of the input domain.
If the quantizer incorporates the spherical structure more directly, the bit budget can be used more efficiently, potentially leading to improved MSE and inner-product performance.

We propose \textit{Block-Sphere Quantization} ($\bq$), which addresses this limitation by constructing centroids that capture richer spherical information at the block level
(Figure~\ref{fig:codebook}). After rotating input vectors, instead of quantizing each
coordinate separately, we group the coordinates of rotated vectors $\zb$ into blocks. Specifically,
for a block size $p$ that divides $d$, we decompose $\zb=(\zb_1,\ldots,\zb_m)$ where $m=d/p$ and $\zb_j\in\mathbb{B}^p$ for $j\in[m]$. Each block $\zb_j$ is then assigned to a centroid in a $p$-dimensional codebook. 
Each block $\zb_j$ is then assigned to a centroid in a $p$-dimensional
codebook. In this way, the quantizer reduces to the coordinate-wise schemes when $p=1$, while larger block sizes allow the codebook to capture more of the spherical structure of the rotated unit vector. The following lemma
characterizes the marginal distribution of each block $\zb_j$ in
$\mathbb{B}^p$.
\begin{lemma}[Block marginal distribution of a uniform spherical vector]\label{lem:block_marginal}
    Suppose $d=mp$ and $\xb \sim \operatorname{Unif}(\sphere)$. Divide vector $\xb$ into blocks $\xb=[\zb_1, \dots, \zb_m]$ where $\zb_j \in \mathbb{B}^p$ for $j \in [m]$. Then, each block $\zb_j$ has density on $\mathbb{B}^p$ : $f_{p,d}(\zb_j)=
\frac{\Gamma(d/2)}{\pi^{p/2}\Gamma((d-p)/2)}
(1-\|\zb_j\|_2^2)^{\frac{d-p-2}{2}}$. Equivalently, $\zb_j$ can decomposed as $\zb_j=r_j\thetab_j$ for $r_j \in [0,1]$, $\thetab_j \in \mathbb{S}^{p-1}$, where $r_j \perp \thetab_j$ and each component follows the distributions: $r_j^2\sim \operatorname{Beta}\!\left(\frac p2,\frac{d-p}{2}\right)$ and $\thetab_j\sim \mathrm{Unif}(S^{p-1})$.
\end{lemma}
The proof of Lemma~\ref{lem:block_marginal} is provided in Appendix~\ref{sec:marginal_density}. Then, we can formulate the block quantization as the following K-means optimization problem: 
\begin{equation}\label{eq:cost}
    \text{(Distortion cost)}=
\int_{B_p}\min_{i\in [2^{bp}]}\|\zb-\ob_i\|_2^2\,f_{p,d}(\zb)\,\operatorname{d}\!\zb.
\end{equation}
Notably, as in $\tq$, the codebook is constructed only once before
quantization and is reused throughout the online quantization procedure.

For the dequantization step in Line~11 of Algorithm~\ref{alg:bq}, we use objective-dependent rescaling rules for $S$, as in $\ed$.
Let $\bar{\xb}$ denote the raw reconstruction before the final rescaling, and let $\rho:=\langle\xb,\bar{\xb}\rangle$. This alignment is computed in the rotated domain in Line~6 of Algorithm~\ref{alg:bq}, since $\langle\xb,\bar{\xb}\rangle=\langle \xb,R^\top\bar{\zb}\rangle=\langle R\xb,\bar{\zb}\rangle=\langle\zb, \bar{\zb}\rangle$.
For $\bqbsm$, the MSE-optimized variant with best scalar rescaling, we set $S=\langle\xb,\bar{\xb}\rangle/\|\bar{\xb}\|_2^2=\rho/\|\bar{\xb}\|_2^2$. For $\bqub$, the unbiased-reconstruction variant, we set $S=1/\rho$. When fast quantization and dequantization are preferred, we also consider the raw-reconstruction variant, denoted by $\bqmse$ following the notation of $\tqmse$. In this case, we set $S=1$, so the value of $\rho$ in Line~6 does not need to be stored.

\begin{algorithm}[t]
\caption{Block-Sphere Quantization ($\bq$)}
\label{alg:bq}
\begin{algorithmic}[1]
\Require dimension $d$, block size $p$, bit-width $b$.
\Statex \textcolor{gray}{// Global Parameters for Setting up $\bq$.}
\State Generate a random rotation matrix $R \in \mathbb{R}^{d \times d}$ from Haar distribution.
\State Construct codebook by finding centroids $\ob_1,\ob_2,\ldots,\ob_{2^{bp}} \in \mathbb{B}^p$ that minimize Equation~\ref{eq:cost}. 

\Statex
\Procedure{\textsc{Quant}}{$\xb$}
    \State $\zb=(\zb_1, \dots, \zb_m) \gets R \mathbf{x}$, where $m=d/p$.
    \State $\mathrm{idx}_j \gets \arg\min_{i \in [2^{bp}]} \|\zb_j - 
    \ob_i\|$ for every $j \in [m]$.
    \State Save the alignment $\rho=\langle \zb, \bar{\zb}\rangle$, where $\bar{\zb}$ aggregated vector of corresponding codewords. 
    \State \Return $\mathrm{idx}=[\mathrm{idx}_1, \ldots, \mathrm{idx}_m] \in \{0,1\}^{b\cdot d}$.
    \Comment{$\mathrm{idx}$'s are $bp$-bit integers}
\EndProcedure

\Statex
\Procedure{\textsc{DeQuant}}{$\mathrm{idx}$}
    \State $\bar{\zb}=[\bar{\zb}_1, \dots,  \bar{\zb}_m]$ where $\bar{\zb}_j \gets c_{\mathrm{idx}_j}$ for every $j \in [m]$.
    \State $\bar{\xb} \gets R^{\top} \bar{\zb}$.
    \State Set the rescaling parameter
\begin{equation*}
S =
\begin{cases}
\dfrac{\rho}{\|\bar{\xb}\|_2^2}, 
& \textnormal{for minimizing MSE (i.e., }\bqbsm\textnormal{)}, \\[1.0em]
\dfrac{1}{\rho}, 
& \textnormal{for unbiased reconstruction (i.e., }\bqub\textnormal{)}, \\[1.0em]
1, 
& \textnormal{for raw reconstruction (i.e., }\bqmse\textnormal{)}.
\end{cases}
\end{equation*}
    \State \Return $S \cdot \tilde{\xb}$.
\EndProcedure
\end{algorithmic}
\end{algorithm}


\subsection{Analysis of $\bq$}\label{subsec:bq_anal}

The following theorem gives MSE bounds for the proposed algorithm with practical block
sizes $p=2$ and $p=3$, as well as for the idealized case $p=d$,  where the centroids lie on $\sphere$. In this case, $f_{d,d}$ in the cost function (Equation~\ref{eq:cost}) is defined by the probability density function of $\text{Unif}(\sphere)$. 

\begin{theorem}[MSE bounds for $\bqmse$ and $\bqbsm$]\label{thm:bq_mse}
    In high dimension, for $\mathcal Q\in\{\bqmse,\bqbsm\}$, $\mse\left(\mathcal Q_{(p=2)}\right)\approx
\mathbf{0.363},\;\mathbf{0.108},\;\mathbf{0.0297},\;\mathbf{0.0078}$ and $\mse\left(\mathcal Q_{(p=3)}\right)\approx
\mathbf{0.357},\;\mathbf{0.101},\;\mathbf{0.0271},\;\mathbf{0.0071}$, for $b=1,~2,~3,~4$, respectively. Moreover, for large $b$, 
$$
\begin{aligned}
    \mse\left(\mathcal Q_{(p=2)}\right) &\le
\mathbf{2.015}\cdot{1 \over {4^b}}(1+o(1)),\\
    \mse\left(\mathcal Q_{(p=3)}\right) &\le
 \mathbf{1.770}\cdot{1 \over {4^b}}(1+o(1)),\\
 \mse\left(\mathcal Q_{(p=d)}\right) &\le C_d\cdot \left( {1\over 4}\right)^{ {bd \over d-1}} \! \cdot (1+o(1)), 
\end{aligned}
$$
where $C_d:=\Gamma\!\left(1+\frac{2}{d-1}\right)\left[2\sqrt\pi\,\frac{\Gamma((d+1)/2)}{\Gamma(d/2)}\right]^{\frac{2}{d-1}} \approx 1.055,~1.008, ~1.001~$ for $~d=100,~1000,~10000$, respectively.
\end{theorem}

\paragraph{Discussion of Theorem~\ref{thm:bq_mse}.}
Since the MSE distortion of $\edbsm$ and $\tqmse$ are  approximately $\mathbf{0.363}, \;\mathbf{0.117}, \;\mathbf{0.0345}, \;\mathbf{0.0095}~$ for $~b=1,~2,~3,~4$ (Proposition~\ref{prop:ed_mse}, Remark~\ref{rmk:tq_mse}), the MSE guarantee of $\bq$ with block size $\ge2$ is better than, that of $\tqmse$ (and other two methods) in both small and large bit regime. Moreover, the result for $p=d$ closely matches the lower bound presented in
Section~\ref{sec:slb}, not only in order but also in the leading constant.
This suggests that increasing the block size in $\bq$ moves the quantizer
toward the optimal quantization limit.
The proof of Theorem~\ref{thm:bq_mse} is deferred to Appendix~\ref{subsec:bq_mse}.

Combining the ratio-estimator bound in Theorem~\ref{thm:ratio_est} with the
MSE bound in Theorem~\ref{thm:bq_mse}, we obtain the bound on
the inner-product error of Algorithm~\ref{alg:bq}.
\begin{corollary}
    [Inner product error bound of Algorithm~\ref{alg:bq}]\label{cor:bq_ip}
In high dimension, the $\bqub$ with block size $p$ satisfies $\iprod\left(\bqub_{(p=2)}\right) \approx
    {\mathbf{0.571} \over d-1},\;{\mathbf{0.120} \over d-1},\;{\mathbf{0.0306} \over d-1},\;{\mathbf{0.0078} \over d-1}$ for $p=2$ and $b=1,~2,~3,~4$, and $\iprod\left(\bqub_{(p=3)}\right)\approx
    {\mathbf{0.553} \over d-1},\;{\mathbf{0.113} \over d-1},\;{\mathbf{0.0279} \over d-1},\;{\mathbf{0.0071} \over d-1}$, for $p=3$ and $b=1,~2,~3,~4$, respectively. Moreover, for large $b$, 
$$
\begin{aligned}
    \iprod\left(\bqub_{(p=2)}\right)
    \le{}&\frac{\mathbf{2.015}}{d-1}4^{-b}(1+o(1)),\\
    \iprod\left(\bqub_{(p=3)}\right)
    \le{}&\frac{\mathbf{1.770}}{d-1}4^{-b}(1+o(1)).
\end{aligned}
$$
\end{corollary}

\paragraph{Discussion of Corollary~\ref{cor:bq_ip}. }Since the MSE distortion of $\edub$ is approximately
$\frac{\mathbf{0.571}}{d-1}$, $\frac{\mathbf{0.133}}{d-1}$,
$\frac{\mathbf{0.0358}}{d-1}$, and $\frac{\mathbf{0.0096}}{d-1}$
for $b=1,~2,~3,~4$, respectively (Corollary~\ref{cor:ed_ip}), the expected inner-product error guarantee for $\bqub$ with block size $\ge2$ is stronger than the corresponding bounds for other existing rotation-based quantizers. The proof of Corollary~\ref{cor:bq_ip} is deferred to Appendix~\ref{subsec:bq_ip}.

\section{Sharper Lower Bound on Distortion}\label{sec:slb}
We derive a lower bound for our quantization problem using the Shannon
lower bound. A similar quantization lower bound based on Shannon's argument
was considered in \citet{zandieh2025turboquant}; however, the entropy term used there is not
directly applicable to a unit vector. (Since
$\sphere$ has zero Lebesgue measure in $\RR^d$, the ambient differential entropy of
$\xb$ is not finite) We correct this by applying the Shannon lower bound to the first $d-1$ coordinates of $\xb$, whose distribution is absolutely continuous on
$\mathbb{B}^{d-1}$. Consequently, unlike the previous expression, the
exponent involves $bd \over (d-1)$ rather than $b$. This correction reflects the
fact that, although the ambient dimension is $d$, the unit sphere $\sphere$
has intrinsic dimension $d-1$.

\begin{theorem}[Shannon distortion lower bound]\label{thm:slb}
    Suppose $\xb \sim \text{Unif}(\sphere)$. Then, for any $b \ge 0$ and any fixed $bd$-bit quantization map $Q$, the MSE is lower bounded as
    \begin{equation*}
        \mathbb{E}_\xb[\|\xb-Q^{-1}(Q(\xb))\|_2^2] \ge c_d\left( {1\over 4}\right)^{ {bd \over d-1}},
    \end{equation*}
    where $c_d:=\frac{d-1}{2\pi e}\left(\frac{\pi^{d/2}}{\Gamma(d/2)}\right)^{2/(d-1)}\exp\!\left(\frac{\psi(1/2)-\psi(d/2)}{d-1}\right) \approx 0.936,0.991,0.999$ for $d=100, 1000, $ $10000$, respectively.
\end{theorem}

\paragraph{Disscusion on Theorem~\ref{thm:slb}. } 
The resulting theorem can be interpreted as an $\mse$ lower bound for rotation-based quantizers. Indeed, after a random rotation, the rotated
input vector can be viewed as $\zb\sim\rm{Unif}(\sphere)$, and
nearest-centroid quantization with a fixed codebook can be regarded as a
fixed quantization map on the sphere. Since the distance between two points on sphere is preserved under rotation, the left hand side is equal to $\mse$. This lower bound
shows the tightness of Theorem~\ref{thm:bq_mse}. Moreover, the comparison suggests that, as
the block size $p$ increases, $\bq$ moves closer to the optimal
quantization limit. The proof of Theorem\ref{thm:slb} is provided in  Appendix~\ref{sec:lb}.

\section{Experiments}\label{sec:exp}
To examine whether the distortion improvements predicted by our theory translate into
practical gains, we evaluate \bq{} on real embedding and LLM inference tasks.
Specifically, we consider three settings: quantization accuracy on real embeddings,
nearest-neighbor search using quantized inner-product estimates, and KV-cache quantization
for long-context LLM inference.
Unless stated otherwise, we use block size $p=3$.
We largely follow the experimental environments and settings of~\citet{zandieh2025qjl}.

\paragraph{Practical implementation.}
Since the exact nearest-centroid assignment for \bq{} requires comparing each
block against all $K=2^{bp}$ centroids, its cost grows with both the block size and the
bit-width. 
This cost becomes non-negligible even for small blocks; for instance, when
$p=3$ and $b=4$, each block has $2^{12}=4096$ candidate centroids. 
To keep the
experiments practical, we use a simple lookup-table approximation: we partition the
block domain into a Cartesian grid, precompute a small set of nearest candidate
centroids for each grid cell, and search only within this candidate set at quantization
time. 
This reduces the online assignment cost from a full codebook search to a small
candidate search, while leaving the codebook construction, dequantization, and
rescaling rules unchanged. 
Further details are provided in Appendix~\ref{app:approx_bq}.
\begin{figure}[t]
    \centering
    \includegraphics[width=\linewidth]{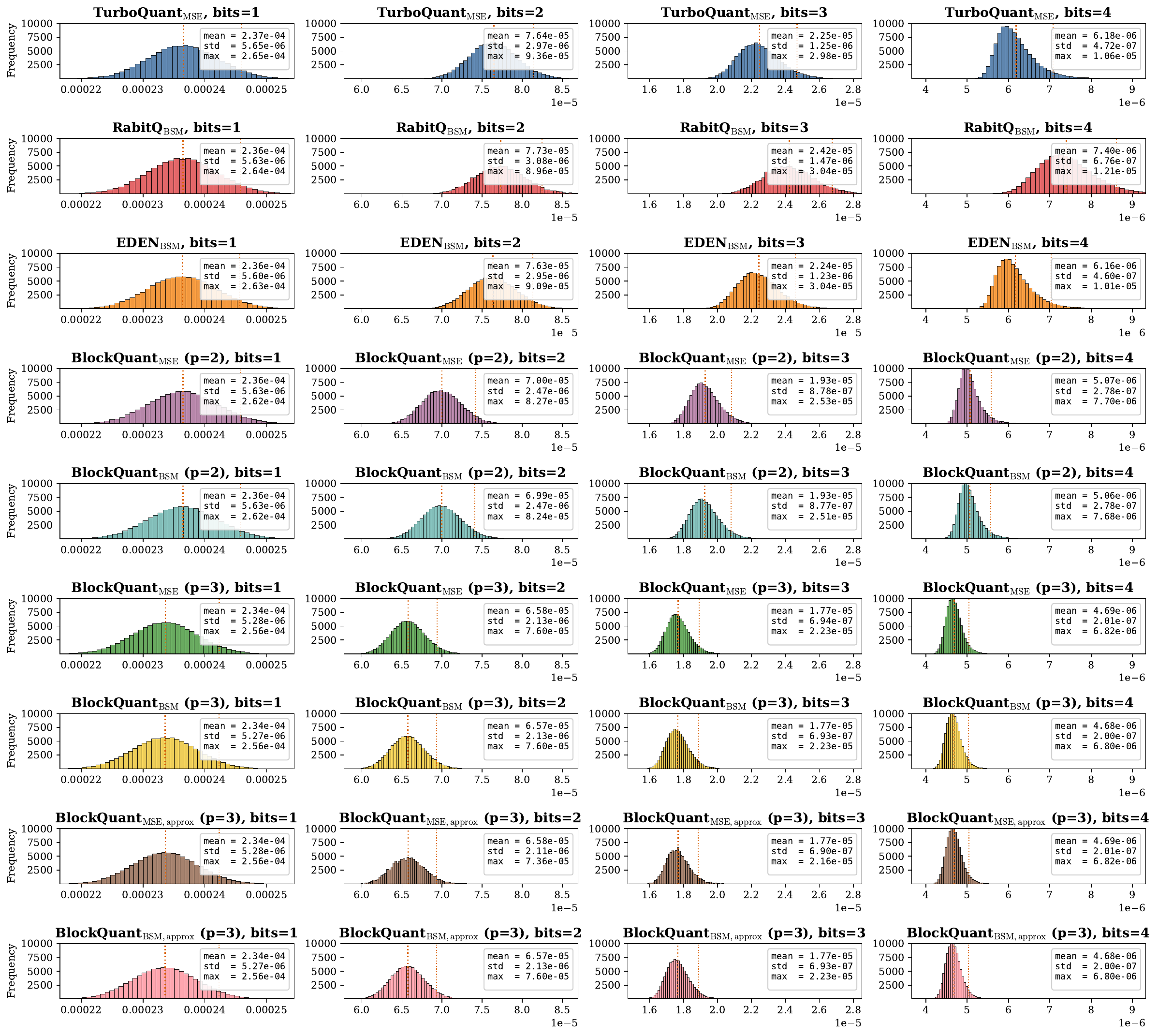}
    \caption{Distribution of MSE.}
    \label{fig:distor_MSE}
\end{figure}
\subsection{Quantization Accuracy.}
\paragraph{Reconstruction MSE.}
We first evaluate reconstruction accuracy on DBpedia Entities~\citep{thakur2021beir} using $1,536$-dimensional embeddings. 
We sample $100,000$ database vectors and normalize them to unit norm, matching the setting of our theoretical analysis. 
For each database vector $\xb_i$, we measure the squared reconstruction error
$e_i=\|\widehat{\xb}_i-\xb_i\|_2^2$, where $\widehat{\xb}_i$ denotes the reconstructed vector.

Figure~\ref{fig:distor_MSE} shows the distribution of the squared reconstruction error. 
Both $\bqmse$ and $\bqbsm$ with block size $p=3$ achieve the smallest distortion among the compared methods, supporting the advantage of block-spherical centroids over coordinate-wise codebooks.
We also observe no noticeable degradation from the approximate nearest-centroid search, suggesting that the approximation preserves the expected-distortion advantage of $\bq{}$ in practice. 
These results further indicate that $\bqbsm$ can serve as an effective reconstruction method in practical settings.

\begin{figure}[t]
    \centering
    \includegraphics[width=\linewidth]{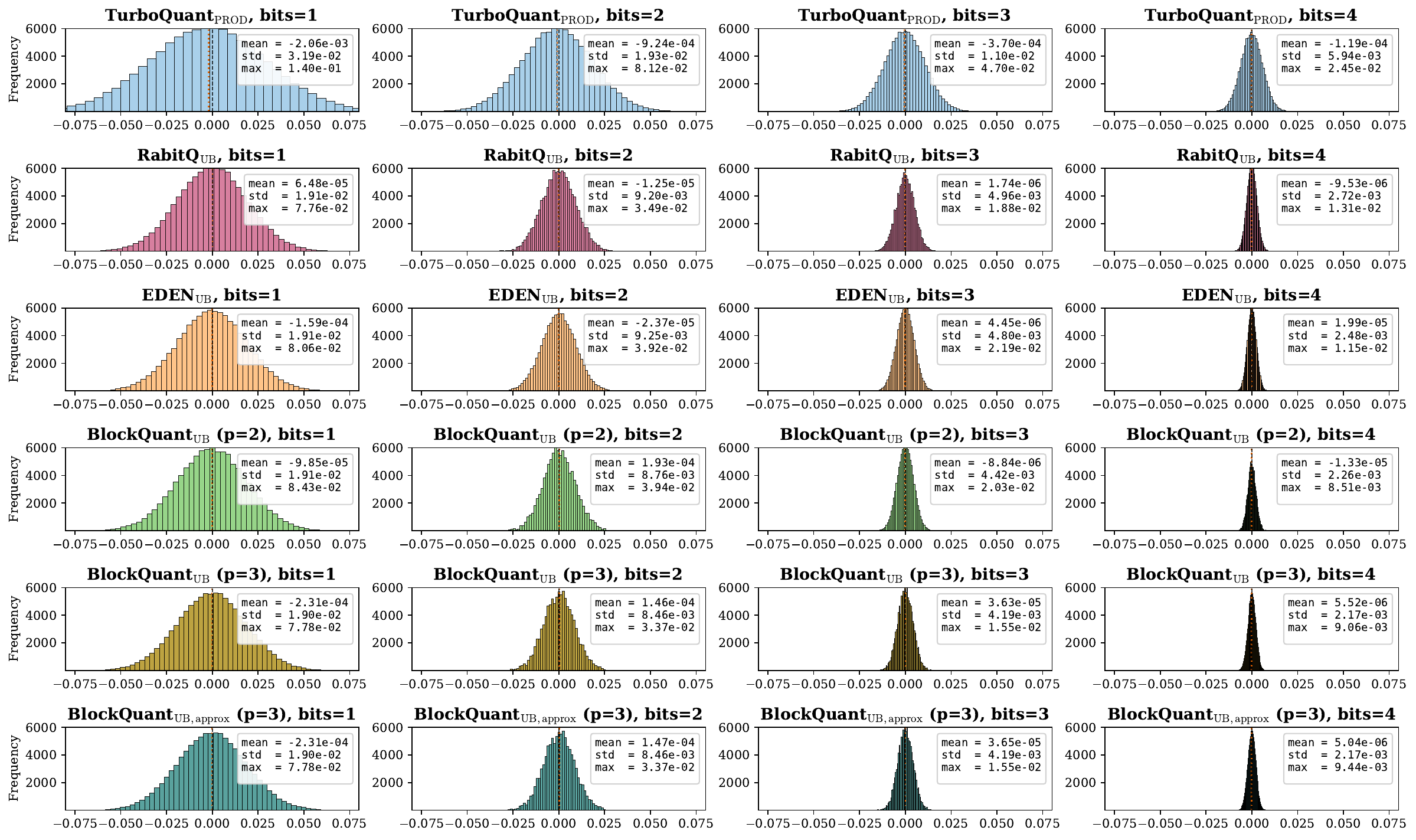}
    \caption{\small{Distribution of inner product error}}
    \label{fig:distor_prod}
\end{figure}
\paragraph{Inner product error.}
We next evaluate inner-product estimation accuracy on DBpedia Entities~\citep{thakur2021beir}
using $1,536$-dimensional embeddings, with $100,000$ database vectors and $1,000$ query vectors. All vectors are normalized, and only the database vectors are quantized. For each pair $(\xb_i,\yb_j)$, we measure the inner-product estimation error $e_{ij}=\langle \widehat{\xb}_i,\yb_j\rangle-\langle \xb_i,\yb_j\rangle$, where $\widehat{\xb}_i$ is the dequantized output of each algorithm.

Figure~\ref{fig:distor_prod} shows that the ratio-based estimators are centered near zero, confirming their empirical unbiasedness. Moreover, $\bqub{}$ yields a more concentrated error distribution than the coordinate-wise baselines, especially at moderate bit-widths.
This is consistent with Corollary~\ref{cor:bq_ip}, which predicts a smaller expected inner-product distortion for the block-spherical construction.

\subsection{Nearest-Neighbor Search.}
We evaluate retrieval quality using Recall@1@k. 
For each query $\qb$, let $g(\qb)$ denote the exact top-$1$ neighbor computed using full-precision inner products, and let $\mathcal{A}_k(\qb)$ denote the set of top-$k$ candidates returned by a method using quantized inner-product estimates.
We define
\begin{equation*}
\mathrm{Recall@1@}k
=
\frac{1}{|\mathcal{Q}|}
\sum_{\qb \in \mathcal{Q}}
\mathbf{1}\!\left\{ g(\qb) \in \mathcal{A}_k(\qb) \right\},
\end{equation*}
where $\mathcal{Q}$ is the query set. Thus, Recall@1@k measures whether quantization preserves the exact nearest neighbor within the top-$k$ retrieved candidates, rather than only measuring average inner-product estimation error.

Figure~\ref{fig:recall_4bit} compares the methods on GloVe ($d=200$) and
OpenAI3/DBpedia ($d=1536$ and $3072$) under $4$-bit compression; the corresponding $2$-bit results are provided in Figure~\ref{fig:recall_2bit}. Across both datasets, the approximate version of $\bqub{}$ achieves strong recall, with the largest gains in the low-$k$ regime. This regime is especially sensitive to quantization error, since small perturbations in inner-product estimates can change the ordering of the top-ranked
candidates.

These results show that the smaller expected inner-product distortion of $\bqub{}$ translates into improved nearest-neighbor retrieval under the same bit budget. In particular, the block-spherical construction improves not only pointwise estimation accuracy but also the ranking quality that is central to approximate nearest-neighbor search.

\subsection{KV Cache Quantization.}
We further evaluate whether the improved distortion of \bq{} translates into end-to-end LLM performance under KV-cache quantization. 
We apply each quantizer to the KV cache of Llama-3.1-8B-Instruct while keeping the model weights unchanged.

\begin{figure}[t]
    \centering
    \includegraphics[width=0.9\linewidth]{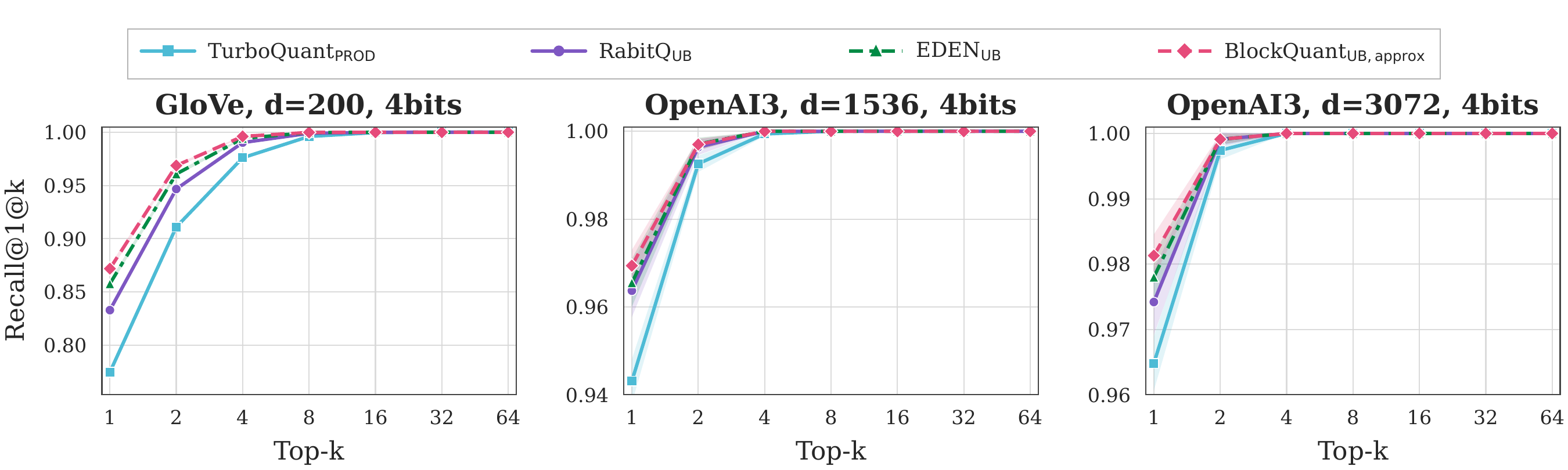}
    \caption{\small{Recall comparison at $4$ bits across different datasets.}}
    \label{fig:recall_4bit}
\end{figure}
\begin{figure}[t]
    \centering
    \includegraphics[width=0.9\linewidth]{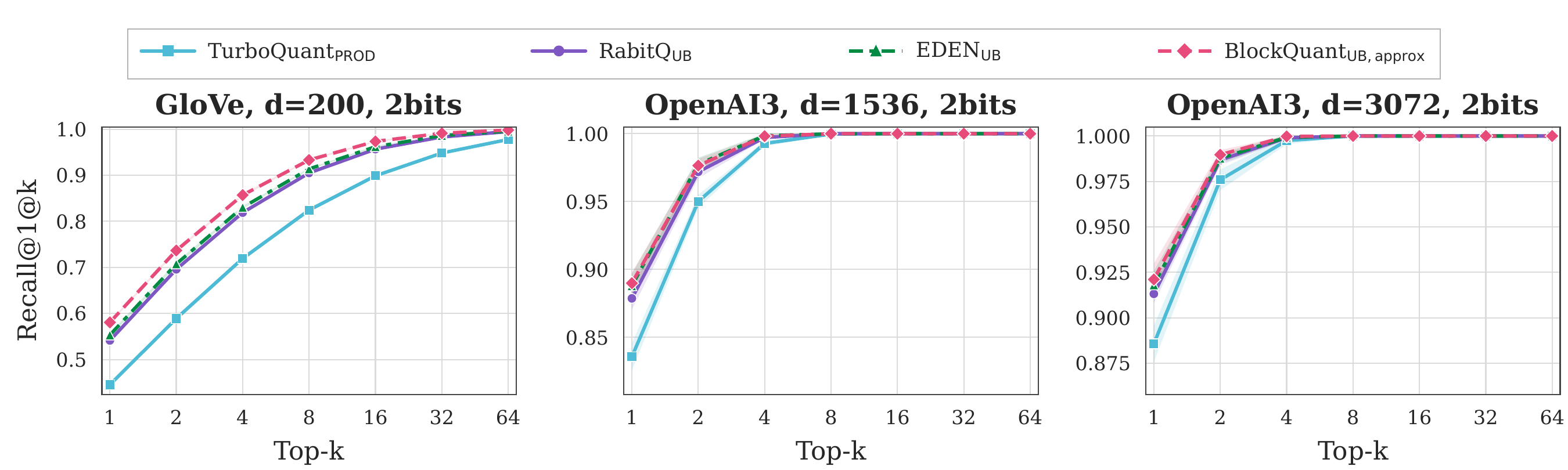}
    \caption{Recall comparison at $2$ bits across different datasets.}
    \label{fig:recall_2bit}
\end{figure}

In the attention computation, the query states are kept in full precision and are not quantized. For the key cache, we follow the outlier-aware configuration used in the KV-cache quantization setup. For each attention head with head dimension $d_h=128$, the $32$ key channels with the largest L2 norm are treated as outlier channels and quantized at a higher bit-width, while the remaining $96$ channels are quantized at a lower bit-width.
In our main $3.5$-bit setting, the outlier channels use $4$-bit quantization and the non-outlier channels use $3$-bit quantization, together with two additional float16 scaling values for the two subvectors. 
This gives an effective key-cache bit-width of $\frac{32 \times 4 + 96 \times 3 + 2 \times 16}{128} = 3.5.$
For the value cache, we quantize the full head dimension uniformly using a $2$-bit approximation. 
All other components, including model weights, MLP layers, embeddings, and output projections, remain unquantized.

We use the approximate nearest-centroid search for $\bq$. We compare against $\rbq$, $\tq$, and $\ed$ under the same KV-cache bit budget, and report results on the Needle-In-A-Haystack~\citep{kamradt2024needle} benchmark 
and LongBench-E~\citep{bai2024longbench}.
Since the randomness induced by sampled rotation matrices has a particularly noticeable effect on LLM inference, we repeat each experiment over five random seeds and report the mean performance with standard deviations.

\begin{figure}[t]
    \centering
    \includegraphics[width=\linewidth]{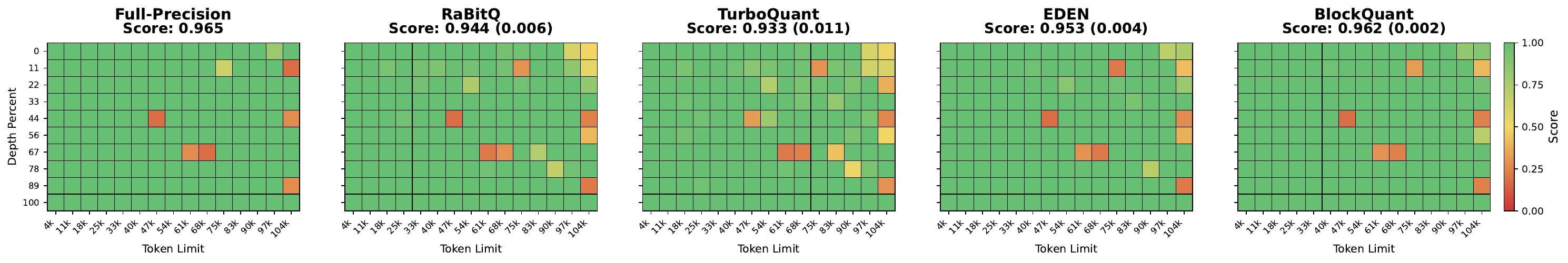}
    \caption{\small{Evaluation of Llama-3.1-8B-Instruct on the ``Needle-In-A-Haystack'' benchmark over five random seeds. Results are reported as mean, with standard deviations shown in parentheses.}}
    \label{fig:needle_3.5bit}
\end{figure}
\begin{table}[t]
\centering
\caption{\small{Evaluation of Llama 3.1 8B Instruct on the ``LongBench-E'' benchmark over five random seeds. Results are reported as mean, with standard deviations shown in parentheses.}}
\label{tab:longbench_35bit}
    \resizebox{\linewidth}{!}{
    \begin{tabular}{l|cccccc|c}
    \toprule
    \textbf{Method} & \textbf{SingleQA} & \textbf{MultiQA} & \textbf{Summ} & \textbf{Few-shot} & \textbf{Synthetic} & \textbf{Code} & \textbf{Average} \\
    
    \midrule
    
    \rbq 
    & 19.12 {\scriptsize (0.54)} 
    & 15.88 {\scriptsize (0.26)} 
    & 29.40 {\scriptsize (0.11)} 
    & 68.23 {\scriptsize (0.10)} 
    & 55.75 {\scriptsize (0.35)} 
    & 60.35 {\scriptsize (0.13)} 
    & 43.52 {\scriptsize (0.15)} \\
    
    \tq 
    & 18.60 {\scriptsize (0.39)} 
    & 15.79 {\scriptsize (0.15)} 
    & 29.11 {\scriptsize (0.11)} 
    & 68.19 {\scriptsize (0.10)} 
    & 55.31 {\scriptsize (0.81)} 
    & 59.72 {\scriptsize (0.49)} 
    & 43.20 {\scriptsize (0.20)} \\
    
    \ed 
    & 19.02 {\scriptsize (0.27)} 
    & 16.30 {\scriptsize (0.26)} 
    & 29.42 {\scriptsize (0.17)} 
    & 68.46 {\scriptsize (0.12)} 
    & 56.29 {\scriptsize (0.22)} 
    & 61.41 {\scriptsize (0.21)} 
    & 43.87 {\scriptsize (0.07)} \\
    
    \bq
    & \textbf{19.55} {\scriptsize (0.46)} 
    & \textbf{16.31} {\scriptsize (0.15)} 
    & \textbf{29.72} {\scriptsize (0.14)} 
    & \textbf{68.49} {\scriptsize (0.16)} 
    & \textbf{56.48} {\scriptsize (0.57)} 
    & \textbf{61.42} {\scriptsize (0.34)} 
    & \textbf{44.03} {\scriptsize (0.10)} \\

    \midrule
    Full Cache (16-bit)
    & 19.53 
    & 16.54 
    & 30.28 
    & 68.41 
    & 56.41 
    & 61.60 
    & 44.15 \\
    
    \bottomrule
    \end{tabular}
    }
\end{table}

\paragraph{Needle-In-A-Haystack.}
Figure~\ref{fig:needle_3.5bit} shows the Needle-In-A-Haystack results across different context lengths and needle depths. The full-precision cache obtains a score of $0.965$. 
Among the quantized methods, $\bq$ achieves the highest score, $0.962$, with a standard deviation of $0.002$, nearly matching the full-precision cache. 
In comparison, $\ed$, $\rbq$, and $\tq$
obtain scores of $0.953$, $0.944$, and $0.933$, respectively. 
This indicates that $\bq$ better preserves the attention-relevant information in the KV cache, especially in long-context settings where small quantization errors can accumulate across many tokens.

\paragraph{LongBench-E.}
Table~\ref{tab:longbench_35bit} reports the LongBench-E results.\footnote{
The original \tq{} evaluation~\citep{zandieh2025qjl} applies additional
prediction post-processing, including truncation to the \textit{first} generated line or
token before scoring. 
To avoid potential evaluation artifacts caused by this truncation, we instead use
the official LongBench-E \texttt{result.json} outputs.
}
$\bq$ achieves the best average score among all quantized methods, with an average of $44.03$, compared with $43.87$ for $\ed$, $43.52$ for $\rbq$, and $43.20$ for $\tq$. 
The score is also close to the full-cache result of $44.15$, leaving only a small gap of $0.12$. 
Moreover, $\bq$ obtains the best quantized performance in all six task groups, including SingleQA, MultiQA, summarization, few-shot tasks, synthetic tasks, and code.

Overall, these results show that the block-spherical construction is effective not only for standalone embedding distortion and nearest-neighbor retrieval, but also for downstream LLM inference. Under the same memory budget, $\bq$ gives the closest performance to the full-precision KV cache and consistently improves over existing rotation-based quantizers.

\section{Conclusion}
In this work, we provide a unified framework for rotation-based quantizers, including $\ed$, $\rbq$, and $\tq$, and compare them from three complementary perspectives: reconstruction MSE, expected inner-product distortion, and high-probability bit complexity. This comparison clarifies that existing methods have different strengths: $\ed$ is favorable for both expected distortion measures, while $\rbq$ provides stronger high-probability guarantees. Motivated by this observation, we propose $\bq$, a block-spherical quantizer that better exploits the geometry of randomly rotated vectors by optimizing centroids at the block level. We prove that $\bq$ improves the distortion constants for both reconstruction MSE and expected inner-product distortion over existing coordinate-wise rotation-based quantizers. Our empirical results support the theoretical findings. 
On embedding retrieval tasks, $\bq$ achieves lower distortion and improved recall compared with existing baselines. We also evaluated KV-cache quantization for Llama-3.1-8B-Instruct, where $\bq$ improves benchmark accuracy over prior rotation-based quantizers. Overall, these results suggest that $\bq$ provides a promising step toward structure-aware vector quantization.

\bibliographystyle{plainnat}
\bibliography{references}

\newpage
\appendix
\counterwithin{table}{section}
\counterwithin{lemma}{section}
\counterwithin{corollary}{section}
\counterwithin{theorem}{section}
\counterwithin{algorithm}{section}
\counterwithin{assumption}{section}
\counterwithin{figure}{section}
\counterwithin{equation}{section}
\counterwithin{condition}{section}
\counterwithin{remark}{section}
\counterwithin{definition}{section}
\counterwithin{proposition}{section}

\clearpage
\appendix

\begingroup
\renewcommand{\partname}{}
\renewcommand{\thepart}{}
\part{Appendix}
\endgroup

\parttoc

\section{Related Work}\label{sec:related_work}

\textbf{Classical vector quantization and high-rate theory. }\,
Vector quantization has a long history in source coding and signal compression. Classical scalar and vector quantizer design is built on the Lloyd--Max optimality conditions and the Linde--Buzo--Gray algorithm, which iteratively alternates nearest-codeword assignment and centroid updates \citep{max1960quantizing,lloyd1982least}. High-rate quantization theory further characterizes the leading-order distortion of optimal vector quantizers through the Zador--Gersho formula \citep{gersho1979asymptotically,zador1982asymptotic}. 

\textbf{Product and block quantization for nearest-neighbor search. }\,
Product quantization (\texttt{PQ}) and its variants are among the most widely used vector quantization methods for approximate nearest-neighbor search. \texttt{PQ} decomposes a high-dimensional vector into low-dimensional subspaces and quantizes each subvector using a separate sub-codebook \citep{jegou2011product, ge2013optimized,norouzi2013cartesian,wang2015optimizedcartesian}. Block-Sphere Quantization ($\bq$) is related to this line of work in that it also quantizes low-dimensional blocks. The key difference is that our codebook is not learned from a dataset. $\bq$ constructs a universal block codebook for this spherical marginal distribution of randomly rotated input vectors in the same framework as $\ed$, $\rbq$, and $\tq$.

\textbf{Spherical quantization. }\,
Several quantization methods exploit hyperspherical geometry by normalizing
vectors and quantizing their directions. For example, deep spherical
quantization uses unit-sphere embeddings for supervised image retrieval,
while binary spherical quantization applies spherical normalization and
binary codes for visual tokenization \citep{eghbali2019deep,zhao2025bsqvit}.
Pyramid vector quantization and related lattice-based methods also use
structured codebooks to represent directions on spherical or pyramidal
domains, and have recently been revisited for neural and LLM compression
\citep{vanderouderaa2024pvq}.

Although these works demonstrate the usefulness of spherical geometry, the
role of the sphere is different from ours. Existing spherical quantizers
typically quantize a full normalized embedding, feature vector, or model-weight
direction, often using a learned or structured spherical codebook. In contrast,
$\bq$ does not construct a codebook on the full sphere. After a random
rotation of $x\in\mathbb{S}^{d-1}$, $\bq$ partitions the rotated vector into
low-dimensional blocks $z_j\in\mathbb{B}^p$ and quantizes the Euclidean-ball
marginal of each block. This marginal distribution is induced exactly by the
unit-sphere geometry and is known in closed form. Thus, $\bq$ uses spherical
geometry through the exact block marginals of randomly rotated unit vectors,
rather than through a global spherical code or a learned hyperspherical
embedding.

\textbf{Rotation-based quantizers and theoretical guarantees. }\,
Several recent quantizers use randomized rotations to make coordinated istributions more regular. \texttt{DRIVE} and $\ed$ apply randomized rotations and scalar quantization for distributed mean estimation and federated learning \citep{vargaftik2021drive,vargaftik2022eden}. $\rbq$combines randomized quantization with a ratio-based estimator and provides high-probability error guarantees for approximate nearest-neighbor search\citep{gao2024rabitq,gao2025practical}. $\tq$ uses the exact coordinate marginal distribution induced by random rotation to design scalar quantizers for MSE, and then combines a reconstruction quantizer with a $\qjl$ residual correction to obtain unbiased inner-product estimation\citep{zandieh2025turboquant,zandieh2025qjl}. Recent notes have compared the $\ed$/\texttt{DRIVE}, $\rbq$, and $\tq$ lines of work, highlighting the need to evaluate these methods under a common set of criteria\citep{benbasat2026note,gao2026revisiting}. However, existing comparisons remain fragmented: they often focus on empirical performance or algorithmic structure, while the theoretical guarantees are stated under different objectives and are therefore difficult to compare directly. Our work follows this comparison-driven perspective by separating three criteria---reconstruction MSE, expected inner-product distortion, and high-probability bit complexity---and analyzing $\ed$, $\rbq$,and $\tq$ under a unified framework. The insights from this comparison then motivate $\bq$, which exploits block-spherical structure to improve expected distortion.

\section{Existing Guarantees on $\ed$, $\rbq$, and $\tq$}
\label{app_sec:existing}
The following proposition is performance guarantees of $\ed$. 
\begin{proposition}[MSE bound of $\edub$, Theorem 2.3 of ~\citet{vargaftik2022eden}]\label{prop:ed_ub_mse}
For all unit input $\xb \in \mathbb{R}^d$, $\edub$ satisfies:
\begin{equation*}
\mse (\edub)
\le
\frac{1}{\mathbb{E}\left[\left(Q(z)\right)^2\right]}
- 1
+ O\left(\sqrt{\frac{\log d}{d}}\right),
\end{equation*}
where $z \sim \mathcal{N}(0,1)$ and $Q$ is Lloyd-max quantizer for standard normal distribution. 
\end{proposition}

The following proposition is performance guarantees of $\rbq$. 
\begin{proposition}[Bit complexity of $\rbq$, Theorem 3.2 of ~\citet{gao2025practical}]\label{prop:rbq_bc}
    For $\epsilon,~ \delta>0$, assume that $\frac{1}{\epsilon^{2}}\log\frac{1}{\delta} > d$ holds. Then, to ensure that the inner product error of the estimator is bounded by $\epsilon$ with the probability of at least $1-\delta$, $\rbq$ requires 
$b=\Theta\!\left( \log\left(
    \frac{1}{d\epsilon^{2}}\log\frac{1}{\delta}
    \right)
    \right)$.
\end{proposition}

The following propositions are performance guarantees of $\tq$. 
\begin{proposition}[Expected distortion bound of $\tqmse$, Theorem 1 of \citet{zandieh2025turboquant}]\label{prop:tq_mse}
    The $b$-bit $\tqmse$ achieves the following distortion rate: 
\begin{itemize}
    \item $\mse (\tqmse)\lesssim {\sqrt{3}\pi \over 2}\cdot{1 \over {4^b}}$ 
    \item For $b=1,~2,~3,~4$, $\mse (\tqmse) \approx 0.36, 0.117, 0.03, 0.009$
\end{itemize}
\end{proposition}

\begin{proposition}[Expected distortion bound of $\tqprod$, Theorem 2 of \citet{zandieh2025turboquant}]\label{prop:tq_ip}
    The $b$-bit $\tqprod$ achieves the following distortion rate: 
\begin{itemize}
    \item $\iprod (\tqprod) \lesssim {\sqrt{3}\pi^2 \over d}\cdot{1 \over {4^b}}\quad$ 
    \item For $b=1,~2,~3,~4$, $\iprod (\tqprod) \approx {1.57 \over d}, {0.56 \over d}, {0.18 \over d}, {0.047 \over d}$. 
\end{itemize}
\end{proposition} 

\section{High-Resolution Formula}\label{sec:high-rate}

In this section we briefly state high-resolution formula of high-rate quantization analysis~\citep{gersho1979asymptotically,zador1982asymptotic}. Let $\Mcal$ be an $m$-dimensional smooth Riemannian manifold with volume measure $\operatorname{d}\! V_\Mcal$, and let $\mu$ be a probability measure on $\Mcal$ with density $f$ with respect to $\operatorname{d}\! V_\Mcal$. 

\begin{definition}[Zador--Gersho constant]
Let $\Tcal\subset \Mcal$ be a bounded measurable subset of $M$ with positive volume, and let $\bar \xb_\Tcal:=\frac1{|\Tcal|}\int_\Tcal \xb\,\operatorname{d}\! V_\Mcal(\xb)$
be its centroid. Its normalized second moment is
\begin{equation*}
G(\Tcal):=\frac{1}{m\,|\Tcal|^{1+2/m}}
\int_\Tcal \|u-\bar u_\Tcal\|_2^2\,\operatorname{d}\! V_\Mcal(\xb).
\end{equation*}
Then, the Zador--Gersho constant $G_\Mcal^{\star}$ is defined as the infimum of $G(\Tcal)$ over all tessellating $m$-dimensional cells:
\begin{equation*}
G_\Mcal^{\star}:=\inf\{G(\Tcal): \Tcal \text{ tiles }\Mcal\text{ by translations}\}.
\end{equation*}
\end{definition}

\begin{definition}[Source factor]
The source factor $J_\Mcal$ is defined as 
\begin{equation*}
J_\Mcal
:=\left(\int_\Mcal f(\xb)^{\frac{m}{m+2}}\,\operatorname{d}\! V_\Mcal(\xb)
\right)^{\frac{m+2}{m}}.
\end{equation*}
\end{definition}

\begin{proposition}[General high-rate Zador--Gersho formula]\label{prop:high-rate}
    Let $\Dcal_\text{MSE}^*(K)$ be the optimal $K$-point quantization distortion for $\Mcal$, i.e. 
    \begin{equation*}
\Dcal_\text{MSE}^*(K):=
\inf_{\mathcal C\subset \RR^d,\, |\mathcal C|\le K}
\int_\Mcal \min_{c\in\mathcal C} \|\xb-c\|_2^2 f(\xb)\,\operatorname{d}\! V_\Mcal(\xb).
\end{equation*}
    Then, the leading-order optimal distortion is
\begin{equation}
\Dcal_\text{MSE}^*(K)=\left(G^*_\Mcal+o(1)\right) J_\Mcal K^{-2/m}
,
\qquad K\to\infty.
\label{eq:zador-gersho}
\end{equation}
\end{proposition}

The Panter--Dite high-resolution formula~\citep{panter2006quantization} used for scalar fixed-rate quantization is the $m=1$ specialization of eq.~\ref{eq:zador-gersho}. For a one-dimensional density $f$ and a scalar quantizer with $K$ levels,
\begin{equation*}
\Dcal_\text{MSE}^*(K)=\left(\frac{1}{12}+o(1)\right)
\left(\int f(x)^{1/3}\,\operatorname{d}\! x\right)^3K^{-2},
\qquad K\to\infty.
\label{eq:panter-dite-K}
\end{equation*}


\section{MSE Analysis}\label{sec:mse}
\subsection{Proof of Proposition~\ref{prop:ed_mse}}\label{subsec:ed_mse}

\subsubsection{Proof for Small $b=1,~2,~3,~4$}
\begin{proof}
Let $\zb=R\xb$, let $R_j:=\sqrt d\,z_j$, and set $K=2^b$. Let $Q_b:\mathbb{R}\to\{q_1,\ldots,q_K\}$ be the $K$-level Lloyd--Max scalar quantizer for $Z\sim N(0,1)$. Thus there are thresholds $-\infty=t_0<t_1<\cdots<t_K=\infty$ such that $Q_b(r)=q_i$ for $r\in[t_{i-1},t_i)$, $q_i=\mathbb{E}[Z\mid Z\in[t_{i-1},t_i)]$, and $t_i=(q_i+q_{i+1})/2$.
Set $\bar{\zb}:=d^{-1/2}(Q_b(R_1),\ldots,Q_b(R_d))$ and $\bar{\xb}:=R^\top\bar{\zb}$. Since $R$ is orthogonal, $\rho_d:=\langle \xb,\bar{\xb}\rangle=\langle \zb,\bar{\zb}\rangle=d^{-1}\sum_{j=1}^d R_jQ_b(R_j)$ and $\psi_d^2:=\|\bar{\xb}\|_2^2=d^{-1}\sum_{j=1}^d Q_b(R_j)^2$.

The dequantized output of $\edbsm$ uses the reconstruction-optimal scalar
$\alpha^\star=\rho_d/\psi_d^2$. Therefore, for each realization of the random rotation,
\begin{equation}
\label{eq:eden-bsm-small-exact}
    \left\|\xb-\alpha^\star\bar{\xb}\right\|_2^2
    =\min_{\alpha\in\mathbb{R}}\|\xb-\alpha\bar{\xb}\|_2^2
    =1-\frac{\rho_d^2}{\psi_d^2}.
\end{equation}

For fixed $b$, the empirical averages in $\rho_d$ and $\psi_d^2$ are approximated in high dimension by their Gaussian counterparts:
\begin{equation*}
    \rho_d
    =\mathbb{E}[ZQ_b(Z)]+o_d(1),
    \qquad
    \psi_d^2
    =\mathbb{E}[Q_b(Z)^2]+o_d(1).
\end{equation*}
By the Lloyd--Max centroid condition,
\begin{equation*}
    \mathbb{E}[Z\mid Q_b(Z)]=Q_b(Z),
    \qquad
    \mathbb{E}[ZQ_b(Z)]=\mathbb{E}[Q_b(Z)^2]=:m_b.
\end{equation*}
Substituting these approximations into \eqref{eq:eden-bsm-small-exact} gives
\begin{equation*}
    \mathbb{E}_{R}\left[\left\|\xb-\alpha^\star\bar{\xb}\right\|_2^2\right]
    \approx
    1-\frac{m_b^2}{m_b}
    =1-m_b.
\end{equation*}
Equivalently, if $e_b:=\mathbb{E}\left[(Z-Q_b(Z))^2\right]$, then the centroid condition also gives $e_b=1-m_b$, and hence the high-dimensional approximate MSE of $\edbsm$ is $e_b$ (which is also the approximate MSE of $\tqmse$).
For $b=1$, $Q_1(Z)=\sqrt{2/\pi}\operatorname{sgn}(Z)$, so $m_1=2/\pi$ and $e_1=1-2/\pi\approx0.363380$.
For $b=2,3,4$, solving the Lloyd--Max equations gives the following scalar distortions.
\begin{center}
\begin{tabular}{rrrr}
\toprule
$b$ & $K=2^b$ & $m_b=\mathbb{E}[Q_b(Z)^2]$ & $e_b=1-m_b$ \\
\midrule
$1$ & $2$  & $0.6366197724$ & $0.3633802276$ \\
$2$ & $4$  & $0.8825181522$ & $0.1174818478$ \\
$3$ & $8$  & $0.9654522392$ & $0.0345477608$ \\
$4$ & $16$ & $0.9904989920$ & $0.0095010080$ \\
\bottomrule
\end{tabular}
\end{center}
Thus, in high dimensions,
\begin{equation*}
    \mse(\edbsm)
    \approx
    0.363,\;0.117,\;0.0345,\;0.0095
\end{equation*}
for $b=1,~2,~3,~4$, respectively.
\end{proof}

\subsubsection{Proof for Large Bit-Width}
\begin{proof}
Since $\edbsm$ chooses the best scalar multiple of the raw reconstruction $\bar{\xb}$, we have, for every realization of
the random rotation,
\begin{equation*}
    \| \xb-\widehat{\xb}_{\edbsm}\|_2^2=
\min_{\alpha\in\mathbb R}\|\xb-\alpha\bar{\xb}\|_2^2\le\|\xb-\bar{\xb}\|_2^2.
\end{equation*}

In the high-dimensional approximation, the scaled rotated
coordinates are asymptotically standard normal.
Therefore, when $d \to \infty$, 
\begin{equation*}
    \mse(\edbsm)\le\mathbb E[(Z-Q_b(Z))^2],
\end{equation*}
where $Z\sim N(0,1)$ and $Q_b$ is the $b$-bit Gaussian Lloyd--Max quantizer. By the
Panter--Dite high-rate formula (Appendix~\ref{sec:high-rate}),
\begin{equation*}
   \mse(\edbsm)\le\mathbb E[(Z-Q_b(Z))^2]
\le
\frac{\pi\sqrt3}{2}4^{-b}(1+o_b(1)),
\end{equation*}
Hence, in high-demensional, $\edbsm$ and $\tqmse$ has the same leading constant. 
 
We note that if high-rate analysis based on the exact spherical marginal density
\begin{equation*}
    f_{1,d}(s)=\frac{\Gamma(d/2)}{\sqrt{\pi}\,\Gamma((d-1)/2)}(1-s^2)^{(d-3)/2},
    \qquad -1\le s\le1,
\end{equation*}
is applied for $\edbsm$, the coefficient of $4^{-b}$ would be $\frac{d+3}{d+6}\frac{\pi d}{12}
    \frac{\Gamma(d/2)}{\Gamma((d-1)/2)}
    \left[\frac{\Gamma((d+3)/6)}{\Gamma((d+6)/6)}
    \right]^3$, which converges to $\pi\sqrt3/2$ as $d\to\infty$.
\end{proof}

\subsection{Proof of Proposition~\ref{prop:rbq_mse}}\label{subsec:rbq_mse}
\begin{proof}
Let $\bar{\xb}$ be the selected $\rbq$ codeword with $\|\xb\|_2=1$ and set $\rho:=\langle\xb,\bar{\xb}\rangle$. The MSE-best scalar reconstruction along $\bar{\xb}$ is $\hat{\xb}:=\rho\bar{\xb}$, and since $\|\xb\|_2=\|\bar{\xb}\|_2=1$,
\begin{equation}\label{eq:rbq_mse_est}
    \|\xb-\hat{\xb}\|_2^2=\|\xb-\rho\bar{\xb}\|_2^2=1-\rho^2.
\end{equation}
Let $\mathcal{G}_b:=\{-(2^b-1)/2+u:u=0,1,\ldots,2^b-1\}^d$. The unrotated $\rbq$ codebook is $\mathcal{C}_{\mathrm{RabitQ}}=\{\mathbf g/\|\mathbf g\|_2:\mathbf g\in\mathcal{G}_b\}$, and the algorithm selects the codeword closest to $\zb$. Therefore the squared angular error is
\begin{equation*}
    1-\rho^2=\min_{t>0,\,\mathbf g\in\mathcal{G}_b}\|\zb-t\mathbf g\|_2^2.
\end{equation*}
For $\alpha=t\sqrt d$, this objective becomes
\begin{equation*}
    \|\zb-t\mathbf g\|_2^2=\sum_{j=1}^d(z_j-tg_j)^2=\frac1d\sum_{j=1}^d(R_j-\alpha g_j)^2.
\end{equation*}
For fixed $\alpha$, the minimizing $g_j$ is the nearest grid point to $R_j/\alpha$, namely $g_j=Q_b(R_j/\alpha)$. Combining this with \eqref{eq:rbq_mse_est} gives
\begin{equation}\label{eq:rbq_mse}
    \mse(\rbqbsm)=\mathbb{E}[1-\rho^2]
    =\mathbb{E}\left[\min_{\alpha>0}\frac1d\sum_{j=1}^d\left(R_j-\alpha Q_b(R_j/\alpha)\right)^2\right].
\end{equation}

For large $d$, the rescaled coordinates $R_j=\sqrt d\,z_j$ are approximately standard normal in the marginal sense, and empirical averages concentrate. Thus, for $Z\sim N(0,1)$, the right-hand side of \eqref{eq:rbq_mse} is approximated by $\min_{\alpha>0}\phi(\alpha)$, where $\phi(\alpha):=\mathbb{E}[(Z-\alpha Q_b(Z/\alpha))^2]$. If $\zeta$ denotes the standard normal density, symmetry gives
\begin{equation}\label{eq:eb_integral_general_new_notation}
    \phi(\alpha)
    =2\sum_{k=0}^{2^{b-1}-2}\int_{k\alpha}^{(k+1)\alpha}\left(r-\alpha\left(k+\frac12\right)\right)^2\zeta(r)\,dr
    +2\int_{(2^{b-1}-1)\alpha}^{\infty}\left(r-\alpha\left(2^{b-1}-\frac12\right)\right)^2\zeta(r)\,dr.
\end{equation}
For $b=1$, $Q_1(Z/\alpha)=\frac12\operatorname{sgn}(Z)$, so $\phi(\alpha)=1-\alpha\sqrt{2/\pi}+\alpha^2/4$ and the minimum is $1-2/\pi\approx0.363380$ at $\alpha=2\sqrt{2/\pi}$. For $b=2,3,4$, numerical minimization of \eqref{eq:eb_integral_general_new_notation} gives the following values.
\begin{center}
\begin{tabular}{ccc}
\toprule
$b$ & minimizing $\alpha$ & $\min_{\alpha>0}\phi(\alpha)$ \\
\midrule
$1$ & $1.5957691216$ & $0.3633802276$ \\
$2$ & $0.9956867007$ & $0.1188460504$ \\
$3$ & $0.5860194285$ & $0.0374396594$ \\
$4$ & $0.3352006088$ & $0.0115428844$ \\
\bottomrule
\end{tabular}
\end{center}
These are the stated Gaussian-approximation constants.
\end{proof}

\subsection{Approximated MSE of $\tq$}\label{subsec:tq_mse}
Let $\mathbf{z}=R\mathbf{x}$ be the randomly
rotated input vector and define the rescaled coordinate
$R_j:=\sqrt d\,z_j$. TurboQuant applies coordinate-wise scalar
quantization, where the $2^b$ centroids are chosen to minimize the
one-dimensional MSE distortion. In the high-dimensional regime, each
$R_j$ is well approximated by $Z\sim N(0,1)$. Hence the $b$-bit
scalar codebook can be computed from the Lloyd--Max problem
\begin{equation*}
    \mathcal A_b^{\rm TQ}:=\arg\min_{\substack{\mathcal A\subset\mathbb R\\|\mathcal A|=2^b}}\mathbb E_{Z\sim N(0,1)}\left[\min_{a\in\mathcal A}(Z-a)^2\right]
\end{equation*}
which is identical to the original codebook of $\ed$~\citep{vargaftik2022eden}.
If $q_b$ denotes the nearest-centroid map associated with
$\mathcal A_b^{\rm TQ}$, then the dequantized vector is
\begin{equation*}
    \bar{\mathbf{x}}\approx R^\top \frac{1}{\sqrt d} \bigl(q_b(R_1),\ldots,q_b(R_d)\bigr),
\end{equation*}
and concentration of empirical averages gives
\begin{align*}
     \mse(\tqmse)=\mathbb E\|\mathbf{x}-\bar{\mathbf{x}}\|_2^2
    \approx \mathbb E_{Z\sim N(0,1)} \left[(Z-q_b(Z))^2\right]
\end{align*}
Solving this scalar Lloyd--Max problem gives
\begin{equation*}
    \mse(\tqmse) \approx 0.363380,\;0.117482,\;0.034548,\;0.009501,
\end{equation*}
for $b=1,~2,~3,~4$, respectively.

\section{Inner Product Distortion Analysis}\label{sec:IP}

\subsection{Proof of Theorem~\ref{thm:ratio_est}}\label{subsec:ratio_est}
\begin{proof}

We first record the isotropy supplied by the random rotation. Fix $\xb$ with $\|\xb\|_2=1$, and consider any orthogonal matrix $U$ satisfying $U\xb=\xb$, i.e., any rotation in the stabilizer of $\xb$. By Haar invariance, $RU^\top\stackrel{d}{=}R$. Moreover, since $U^\top\xb=\xb$, the rotated input is unchanged, $(RU^\top)\xb=R\xb$, while the final inverse rotation gives $\bar\xb_{RU^\top}=U\bar\xb_R$. Consequently, $\rho:=\langle\bar\xb,\xb\rangle$ is unchanged, since $\langle U\bar\xb_R,\xb\rangle=\langle\bar\xb_R,U^\top\xb\rangle=\langle\bar\xb_R,\xb\rangle$, and $\psi:=\|\bar\xb\|_2$ is also unchanged. On the other hand, the residual component $\vb:=\bar\xb-\rho\xb$, which lies in $\xb^\perp$, is transformed as $\vb\mapsto U\vb$. Thus, after conditioning on the scalar quantities $(\rho,\psi)$, the only remaining randomness in $\vb$ is its direction inside $\xb^\perp$. Since the stabilizer of $\xb$ acts transitively on directions in $\xb^\perp$, the conditional distribution of $\vb$ is directionally isotropic in $\xb^\perp$.

Now decompose the reconstruction as $\bar\xb=\rho\xb+\vb$, where $\vb\perp\xb$ by the definition of $\rho=\langle\bar\xb,\xb\rangle$. Since $\|\xb\|_2=1$, this decomposition gives $\|\vb\|_2^2=\|\bar\xb\|_2^2-\rho^2=\psi^2-\rho^2$. Let $\phi:=\langle\xb,\yb\rangle$ and assume $\rho>0$. For the ratio estimator, we have $\widehat\phi_{\rm ratio}=\langle\bar\xb,\yb\rangle/\rho$, and hence $\widehat\phi_{\rm ratio}-\phi=\langle\bar\xb,\yb\rangle/\rho-\langle\xb,\yb\rangle=\langle\vb,\yb\rangle/\rho$. Writing $\yb_\perp:=\yb-\phi\xb$, we have $\yb_\perp\in\xb^\perp$ and $\|\yb_\perp\|_2^2=1-\phi^2$. Moreover, since $\vb\perp\xb$, $\langle\vb,\yb\rangle=\langle\vb,\yb_\perp\rangle$. By the conditional isotropy of $\vb$ in $\xb^\perp$, its conditional mean in every fixed direction of $\xb^\perp$ is zero, and thus $\mathbb{E}[\langle\vb,\yb_\perp\rangle\mid \rho,\psi]=0$. Therefore $\mathbb{E}[\widehat\phi_{\rm ratio}\mid \rho,\psi]=\phi$, so $\widehat\phi_{\rm ratio}$ is unbiased. 

The same conditional isotropy implies that the squared length $\|\vb\|_2^2=\psi^2-\rho^2$ is spread uniformly over the $d-1$ dimensions of $\xb^\perp$, giving $\mathbb{E}[\langle\vb,\yb_\perp\rangle^2\mid \rho,\psi]=(\psi^2-\rho^2)\|\yb_\perp\|_2^2/(d-1)$. Consequently,
\begin{equation*}
    \mathbb{E}\left[(\widehat\phi_{\rm ratio}-\phi)^2\right]
    =\frac{1-\phi^2}{d-1}\mathbb{E}\left[\frac{\psi^2-\rho^2}{\rho^2}\right].
\end{equation*}
\end{proof}


\subsection{Proof of Corollary~\ref{cor:ed_ip}}\label{subsec:ed_ip}

\subsubsection{Proof for Small $b=1,~2,~3,~4$}
\begin{proof}
Let $\zb=R\xb$, $R_j:=\sqrt d\,z_j$, and let $Q_b$ be the $2^b$-level Lloyd--Max scalar quantizer for $Z\sim N(0,1)$. As in the proof of Proposition~\ref{prop:ed_mse}, set $\bar\zb:=d^{-1/2}(Q_b(R_1),\ldots,Q_b(R_d))$ and $\bar\xb:=R^\top\bar\zb$. Then $\rho_d:=\langle\xb,\bar\xb\rangle=d^{-1}\sum_{j=1}^d R_jQ_b(R_j)$ and $\psi_d^2:=\|\bar\xb\|_2^2=d^{-1}\sum_{j=1}^d Q_b(R_j)^2$.

The $\edub$ estimator is $\widehat\eta_{\rm ratio}:=\langle\bar\xb,\yb\rangle/\rho_d$. Applying Theorem~\ref{thm:ratio_est} with $\rho=\rho_d$ and $\psi=\psi_d$, for $\eta=\langle\xb,\yb\rangle$ we get
\begin{equation*}
    \mathbb{E}\left[(\widehat\eta_{\rm ratio}-\eta)^2\right]
=\frac{1-\eta^2}{d-1}\mathbb{E}\left[\frac{\psi_d^2}{\rho_d^2}-1\right].
\end{equation*}
Taking the supremum over $\yb\in\sphere$ gives $\iprod(\edub)\le (d-1)^{-1}(\mathbb{E}[\psi_d^2/\rho_d^2]-1)$ ($=$ holds when $\eta=0$). The concentration argument used in Proposition~\ref{prop:ed_mse} gives $\mathbb{E}[\psi_d^2/\rho_d^2]\le m_b^{-1}+O(\sqrt{\log d/d})$, where $m_b:=\mathbb{E}[Q_b(Z)^2]$. Hence
\begin{equation*}
    \iprod(\edub)
\le \frac{1}{d-1}\left(\frac1{m_b}-1+O\!\left(\sqrt{\frac{\log d}{d}}\right)\right).
\end{equation*}
For Lloyd--Max centroids, $\mathbb{E}[Z\mid Q_b(Z)]=Q_b(Z)$, so $\mathbb{E}[ZQ_b(Z)]=m_b$. Thus, with $e_b:=\mathbb{E}[(Z-Q_b(Z))^2]$, we have $e_b=1-m_b$ and $m_b^{-1}-1=e_b/(1-e_b)=:B_b^{\mathrm{EDEN}}$. From the proof of Proposition~\ref{prop:ed_mse}, the constants are
\begin{center}
\begin{tabular}{rrrr}
\toprule
$b$ & $e_b$ & $m_b=1-e_b$ & $B_b^{\mathrm{EDEN}}=e_b/(1-e_b)$ \\
\midrule
$1$ & $0.3633802276$ & $0.6366197724$ & $0.5707963268$ \\
$2$ & $0.1174818478$ & $0.8825181522$ & $0.1331211687$ \\
$3$ & $0.0345477608$ & $0.9654522392$ & $0.0357840185$ \\
$4$ & $0.0095010080$ & $0.9904989920$ & $0.0095921430$ \\
\bottomrule
\end{tabular}
\end{center}
This proves the stated small-bit inner-product bounds for $\edub$.
\end{proof}

\subsubsection{Proof for Large Bit-Width}
\begin{proof}
Let $\bar\xb$ be the unscaled coordinate-wise high-rate reconstruction used in the proof of Proposition~\ref{prop:ed_mse}. The $\edub$ estimator is $\widehat\eta_{\edub}:=\langle\bar\xb,\yb\rangle/\langle\bar\xb,\xb\rangle$. Write $\eb:=\bar\xb-\xb$, $t:=\langle\xb,\mathbf e\rangle$, and $\ub:=\eb-t\xb$. Then $\ub\perp\xb$, $\bar\xb=(1+t)\xb+\ub$, $\rho_d:=\langle\bar\xb,\xb\rangle=1+t$, and $\psi_d^2-\rho_d^2=\|\ub\|_2^2$.

By Theorem~\ref{thm:ratio_est}, after taking the supremum over $\yb\in\sphere$,
\begin{equation*}
    \iprod(\edub)
=\frac{1}{d-1}\mathbb{E}\left[\frac{\psi_d^2-\rho_d^2}{\rho_d^2}\right]
=\frac{1}{d-1}\mathbb{E}\left[\frac{\|\ub\|_2^2}{(1+t)^2}\right].
\end{equation*}
Thus the ratio correction removes the radial component from the numerator exactly, so the leading contribution comes only from $\|\ub\|_2^2$. In the high-rate regime, $\mathbb{E}\|\eb\|_2^2=O(4^{-b})$ and $\mathbb{E}\|\eb\|_2^4=O(4^{-2b})$. Since $|t|\le\|\eb\|_2$ and $\|\ub\|_2\le\|\eb\|_2$, H\"older's inequality gives $\mathbb{E}\|\ub\|_2^2|t|\le\mathbb{E}\|\eb\|_2^3\le(\mathbb{E}\|\eb\|_2^4)^{3/4}=o(4^{-b})$. Hence the denominator $(1+t)^2$ only changes the expression by a lower-order term, and therefore $\mathbb{E}[\|\ub\|_2^2/(1+t)^2]=\mathbb{E}\|\ub\|_2^2+o(4^{-b})$.

It remains to recall the tangential energy from the high-rate MSE proof. Let $f_{1,d}$ be the one-dimensional marginal density of a coordinate of a uniform vector on $\sphere$, and define $L_d:=\int_{-1}^1 f_{1,d}(s)^{1/3}\,ds$ and $M_d:=\int_{-1}^1s^2f_{1,d}(s)^{1/3}\,ds$. The Panter--Dite formula gives the leading coordinate-wise high-rate distortion, and the radial calculation applies the same local error expansion to the component along $\xb$. Thus,
\begin{equation*}
    \mathbb{E}[\|\eb\|_2^2]=\frac{d}{12}L_d^3 4^{-b}(1+o(1)),\qquad
\mathbb{E}[t^2]=\frac{d}{12}L_d^2M_d 4^{-b}(1+o(1)).
\end{equation*}
Since $\eb=t\xb+\ub$ with $\ub\perp\xb$, the tangential energy is obtained by subtracting the radial energy:
$\mathbb{E}[\|\ub\|_2^2]=\mathbb{E}[\|\eb\|_2^2]-\mathbb{E}[t^2]=\frac{d}{12}L_d^2(L_d-M_d)4^{-b}(1+o(1))$. Also, because $f_{1,d}(s)^{1/3}$ is proportional to $(1-s^2)^{(d-3)/6}$,
\begin{equation*}
    \frac{M_d}{L_d}
=\frac{\int_{-1}^1s^2(1-s^2)^{(d-3)/6}\,ds}{\int_{-1}^1(1-s^2)^{(d-3)/6}\,ds}
=\frac{\mathrm{B}(3/2,(d+3)/6)}{\mathrm{B}(1/2,(d+3)/6)}
=\frac{3}{d+6}.
\end{equation*}
Thus $\mathbb{E}[\|\ub\|_2^2]=\frac{d+3}{d+6}\frac{d}{12}L_d^3 4^{-b}(1+o(1))$. Using $L_d^3=\pi\frac{\Gamma(d/2)}{\Gamma((d-1)/2)}\left[\frac{\Gamma((d+3)/6)}{\Gamma((d+6)/6)}\right]^3$, we obtain
\begin{equation*}
    \iprod(\edub)\le
\frac{1}{d-1}\frac{d+3}{d+6}\frac{\pi d}{12}
\frac{\Gamma(d/2)}{\Gamma((d-1)/2)}
\left[\frac{\Gamma((d+3)/6)}{\Gamma((d+6)/6)}\right]^3
4^{-b}(1+o(1)).
\end{equation*}

Finally, $\Gamma(x+a)/\Gamma(x+b)\sim x^{a-b}$ implies that the coefficient multiplying $4^{-b}/(d-1)$ converges to $\pi\sqrt3/2\approx2.721$. This proves the high-rate inner-product bound for $\edub$.
\end{proof}


\subsection{Proof of Corollary~\ref{cor:rbq_ip}}\label{subsec:rbq_ip}
\begin{proof}
The factor inside the expectation in Theorem~\ref{thm:ratio_est} is nonnegative and independent of $\yb$, so the worst case over $\yb\in\sphere$ is attained when $\phi=0$. For a fixed $\xb$,
\begin{equation*}
    \iprod(\rbq)=\frac{1}{d-1}\mathbb{E}_Q\left[
\frac{\|\bar\xb\|_2^2-\langle\bar\xb,\xb\rangle^2}{\langle\bar\xb,\xb\rangle^2}
\right],
\end{equation*}
and rotation invariance makes this quantity independent of the particular $\xb\in\sphere$.

The codewords of $\rbq$ are normalized, so $\|\bar\xb\|_2=1$. Let $\rho:=\langle\bar\xb,\xb\rangle$ and $\Delta:=1-\rho^2$. As in the proof of Proposition~\ref{prop:rbq_mse}, $\Delta=\|\xb-\rho\bar\xb\|_2^2$ is the reconstruction error of $\rbqbsm$. Hence
\begin{equation}\label{eq:rbq_ip}
\iprod(\rbq)=\frac{1}{d-1}\mathbb{E}_Q\left[\frac{\Delta}{1-\Delta}\right].
\end{equation}
Let $\zb=R\xb$ and $R_j:=\sqrt d\,z_j$. For $Q_b(u):=\operatorname{sign}(u)\min(\lfloor |u|\rfloor+1/2,2^{b-1}-1/2)$, the same MSE proof gives
\begin{equation}\label{eq:delta}
\Delta=\min_{\alpha>0}\frac1d\sum_{j=1}^d\left(R_j-\alpha Q_b(R_j/\alpha)\right)^2.
\end{equation}

We now use the same high-dimensional Gaussian approximation as in the proof of Proposition~\ref{prop:rbq_mse}. Let $Z\sim N(0,1)$ with density $\zeta$. For fixed $b$, the Gaussian limit of the empirical objective in \eqref{eq:delta} is
\begin{align}
\phi(\alpha)
:={}&\mathbb{E}\left[(Z-\alpha Q_b(Z/\alpha))^2\right] \notag \\
={}&2\sum_{k=0}^{2^{b-1}-2}\int_{k\alpha}^{(k+1)\alpha}\left(r-\alpha\left(k+\frac12\right)\right)^2\zeta(r)\,dr
+2\int_{(2^{b-1}-1)\alpha}^{\infty}\left(r-\alpha\left(2^{b-1}-\frac12\right)\right)^2\zeta(r)\,dr,
\label{eq:phi-alpha}
\end{align}
where the summation term is empty when $b=1$. Define $\Delta_z:=\min_{\alpha>0}\phi(\alpha)$ and $\kappa_b^{\mathrm{RQ}}:=\Delta_z/(1-\Delta_z)$. By concentration of the empirical averages, for fixed $b$ and $d\to\infty$, $\Delta=\Delta_z+o_{\mathbb{P}}(1)$. Since $\Delta_z<1$ for the bit-widths considered here and the denominator stays bounded away from zero with high probability, $\mathbb{E}_Q[\Delta/(1-\Delta)]=\kappa_b^{\mathrm{RQ}}+o(1)$. Equation~\eqref{eq:rbq_ip} therefore gives
\begin{equation*}
  \iprod(\rbq)=\frac{\kappa_b^{\mathrm{RQ}}+o(1)}{d-1}.
  \end{equation*}
The minimization in \eqref{eq:phi-alpha}, computed in Proposition~\ref{prop:rbq_mse}, gives
\begin{center}
\begin{tabular}{c|c|c|c}
$b$ & Optimal value of $\alpha$ & $\Delta_z$ & $\kappa_b^{\mathrm{RQ}}=\Delta_z/(1-\Delta_z)$ \\
\hline
$1$ & $1.595769$ & $0.363380$ & $0.570796$ \\
$2$ & $0.995687$ & $0.118846$ & $0.134875$ \\
$3$ & $0.586019$ & $0.037440$ & $0.038896$ \\
$4$ & $0.335201$ & $0.011543$ & $0.011678$
\end{tabular}
\end{center}
Thus, for $b=1,~2,~3,~4$, the leading constants are $0.570796/(d-1)$, $0.134875/(d-1)$, $0.038896/(d-1)$, and $0.011678/(d-1)$, respectively.
\end{proof}


\section{Bit Complexity Analysis (Proof of Theorem~\ref{thm:ed_tq_bc})}\label{sec:bc}

We prove the bit-complexity guarantees for $\edub$ and $\tqprod$ separately in Sections~\ref{subsec:ed_bc} and~\ref{subsec:tq_bc}. Each subsection contains the corresponding formal statement, namely Theorems~\ref{thm:ed_bc} and~\ref{thm:tq_bc}, the high-probability residual bound, and the proof.

\subsection{Bit Complexity of $\ed$}\label{subsec:ed_bc}
\begin{theorem}[Bit complexity of $\edub$ (formal)]
\label{thm:ed_bc}
Let $\xb,\yb\in\sphere$ and let $0<\epsilon,\delta<1$. There exist universal constants $C,c>0$ such that, if $\log(4/\delta)\le cd$, then for every bit-width $b\ge1$,
\begin{equation*}
    \mathbb{P}_{Q}\left[\left|{\rm IP}_{\rm EDEN}^{(b)}(\xb,\yb)-\langle\xb,\yb\rangle\right|>
    C\left(2^{-b}\sqrt{\frac{\log(4/\delta)}{d}}+\frac{\log(4/\delta)}{d}
    \right)\right]
    \le \delta .
\end{equation*}
Consequently, whenever $\epsilon\ge C\log(4/\delta)/d$, to ensure
\begin{equation*}
    \mathbb{P}_{Q}\left[\left|{\rm IP}_{\rm EDEN}^{(b)}(\xb,\yb)-\langle\xb,\yb\rangle\right|>\epsilon\right]
    \le \delta,
\end{equation*}
it is sufficient to take
\begin{equation*}
    b\ge\max\left\{1,\left\lceil\frac{1}{2}\log_2\left( \frac{C\log(4/\delta)}{d\epsilon^2}\right)\right\rceil\right\}
\end{equation*}
bits per dimension.
\end{theorem}

\subsubsection{High probability residual bound}
\begin{lemma}[High probability residual bound $\ed$]
\label{lem:ed_hp}
Assume that the unscaled $b$-bit $\ed$ reconstruction satisfies
\begin{equation*}
    \max_{\xb\in\sphere}
    \mathbb{E}_{Q}\left[\|\xb-\bar\xb\|_2^2\right]
    \le C_{\rm ed}4^{-b}.
\end{equation*}
Let $\xb,\yb\in\sphere$ be fixed, and let $\bar\xb$ be the unscaled dequantized codeword produced by the $b$-bit $\ed$ quantizer. Then, for every $0<\delta<1$,
\begin{equation}
    \mathbb{P}_{Q}\left[\|\xb-\bar\xb\|_2>\sqrt{C_{\rm ed}}\,2^{-b}
    +C_{\rm levy}\sqrt{\frac{\log(4/\delta)}{d}}\right]
    \le \frac{\delta}{2},
    \label{eq:ed_residual_hp}
\end{equation}
where $C_{\rm levy}>0$ is the universal constant from L\'evy's concentration.
\end{lemma}
\begin{proof}
Fix $b,d$ and $\xb\in\sphere$. Let $R$ be the random rotation, write $\zb=R\xb$, and let $\mathcal{C}_{\rm EDEN}^{(b)}$ be the fixed rotated-coordinate product codebook after the $1/\sqrt d$ scaling. If $\bar\zb\in\arg\min_{\ob_i\in\mathcal{C}_{\rm EDEN}^{(b)}}\|\zb-\ob_i\|_2$ and $\bar\xb=R^\top\bar\zb$, then $\|\xb-\bar\xb\|_2=\|\zb-\bar\zb\|_2$ by orthogonality of $R$. Define $f(\zb):=\operatorname{dist}(\zb,\mathcal{C}_{\rm EDEN}^{(b)})=\min_{\ob_i\in\mathcal{C}_{\rm EDEN}^{(b)}}\|\zb-\ob_i\|_2$, so $f(\zb)=\|\xb-\bar\xb\|_2$.

The distance-to-a-set map is $1$-Lipschitz: for any $\zb,\zb'\in\sphere$, $f(\zb)-f(\zb')\le\|\zb-\ob_{i^\star(\zb')}\|_2-\|\zb'-\ob_{i^\star(\zb')}\|_2\le\|\zb-\zb'\|_2$, and reversing $\zb,\zb'$ gives $|f(\zb)-f(\zb')|\le\|\zb-\zb'\|_2$. Set $L=\log(4/\delta)$. By L\'evy's concentration inequality, after increasing $C_{\rm levy}$ if necessary,
\begin{equation*}
    \mathbb{P}_{Q}\left[f(\zb)>\mathbb{E}_{Q}f(\zb)+C_{\rm levy}\sqrt{\frac{L}{d}}\right]
    \le \frac{\delta}{2}.
\end{equation*}
Moreover, Jensen's inequality and the assumed MSE bound give $\mathbb{E}_{Q}f(\zb)\le(\mathbb{E}_{Q}f(\zb)^2)^{1/2}=(\mathbb{E}_{Q}\|\xb-\bar\xb\|_2^2)^{1/2}\le\sqrt{C_{\rm ed}}\,2^{-b}$. Combining the two estimates proves \eqref{eq:ed_residual_hp}.
\end{proof}

\subsubsection{Proof of Theorem~\ref{thm:ed_bc}}
Fix $b\ge1$. Let $\bar\xb$ be the unscaled $b$-bit $\ed$ codeword, set $\rho=\langle\xb,\bar\xb\rangle$, and define $\widehat\xb_{\rm EDEN}^{(b)}=\bar\xb/\rho$. The $\edub$ estimator is ${\rm IP}_{\rm EDEN}^{(b)}(\xb,\yb):=\langle\widehat\xb_{\rm EDEN}^{(b)},\yb\rangle=\langle\bar\xb,\yb\rangle/\rho$. Let $\eta=\langle\xb,\yb\rangle$, $\vb=\bar\xb-\rho\xb$, and $\yb_\perp=\yb-\eta\xb$. Then $\vb\perp\xb$, $\yb_\perp\perp\xb$, and
\begin{equation*}
    {\rm IP}_{\rm EDEN}^{(b)}(\xb,\yb)-\langle\xb,\yb\rangle
    =\frac{\langle\vb,\yb\rangle}{\rho}
    =\frac{\langle\vb,\yb_\perp\rangle}{\rho}.
\end{equation*}

We next record the isotropy supplied by the Haar rotation. Write the EDEN output as $\bar\xb_R=R^\top Q_0(R\xb)$, where $Q_0$ is the fixed rotated-coordinate scalar quantize--dequantize map. If $U\xb=\xb$, then $RU^\top\stackrel{d}{=}R$ and $\bar\xb_{RU^\top}=(RU^\top)^\top Q_0(RU^\top\xb)=U\bar\xb_R$. Hence $\rho$ and $\|\vb\|_2$ are invariant, while $\vb$ is transformed into $U\vb$. Therefore, conditional on $\rho$ and $\|\vb\|_2$, the direction of $\vb$ is rotationally invariant in $\xb^\perp$.

Set $L=\log(4/\delta)$ and $\alpha_b:=\sqrt{C_{\rm ed}}\,2^{-b}+C_{\rm levy}\sqrt{L/d}$. By Lemma~\ref{lem:ed_hp}, with probability at least $1-\delta/2$, $\|\xb-\bar\xb\|_2\le\alpha_b$. The EDEN MSE bound from the previous section gives a universal constant $C_{\rm ed}<4$ for the unscaled reconstruction. Hence, by choosing the universal constant $c$ small enough, $L\le cd$ implies $\alpha_b\le\alpha_0$ for some fixed $\alpha_0<1$ and all $b\ge1$. On this residual event, $\|\vb\|_2\le\|\bar\xb-\xb\|_2\le\alpha_b$, and since $\|\bar\xb-\xb\|_2^2=\|\bar\xb\|_2^2+1-2\rho$ with $\|\bar\xb\|_2^2\ge0$, we also have $\rho\ge(1-\alpha_b^2)/2\ge(1-\alpha_0^2)/2=:c_\rho>0$.

Conditional on $\rho$ and $\|\vb\|_2$, write $\vb=\|\vb\|_2\boldsymbol{\theta}$ with $\boldsymbol{\theta}$ uniform on the unit sphere in $\xb^\perp$. For fixed $\yb_\perp$, L\'evy's concentration on this $(d-2)$-dimensional sphere gives universal constants $C_{\rm sph},c_{\rm sph}>0$ such that, whenever $L\le c_{\rm sph}d$,
\begin{equation*}
    \mathbb{P}_{Q}\left[
    |\langle\vb,\yb_\perp\rangle|>
    C_{\rm sph}\|\vb\|_2\|\yb_\perp\|_2\sqrt{\frac{L}{d}}
    \;\middle|\;\rho,\|\vb\|_2
    \right]
    \le \frac{\delta}{2}.
\end{equation*}
Since $\|\yb_\perp\|_2\le1$, a union bound yields, with probability at least $1-\delta$,
\begin{equation*}
    \left|{\rm IP}_{\rm EDEN}^{(b)}(\xb,\yb)-\langle\xb,\yb\rangle\right|
    \le \frac{C_{\rm sph}}{c_\rho}\alpha_b\sqrt{\frac{L}{d}}
    \le C\left(2^{-b}\sqrt{\frac{L}{d}}+\frac{L}{d}\right).
\end{equation*}
This proves the high-probability bound after taking $c\le c_{\rm sph}$ and replacing $L$ by $\log(4/\delta)$.

It remains to choose $b$ for target accuracy $\epsilon$. If $\epsilon\ge C_0L/d$ with $C_0\ge2C$, then $CL/d\le\epsilon/2$. If
\begin{equation*}
    b\ge
    \max\left\{
    1,
    \left\lceil
    \frac{1}{2}\log_2\left(
    \frac{C_1L}{d\epsilon^2}
    \right)
    \right\rceil
    \right\},
\end{equation*}
then, after increasing $C_1$ if necessary, $C2^{-b}\sqrt{L/d}\le\epsilon/2$: if $C_1L/(d\epsilon^2)>1$, this follows from the lower bound on $b$, while if $C_1L/(d\epsilon^2)\le1$, it follows from $\sqrt{L/d}\le\epsilon/\sqrt{C_1}$ and $2^{-b}\le1$. Thus the error is at most $\epsilon$ with probability at least $1-\delta$, and enlarging the theorem constant $C$ to dominate $C_0,C_1$ completes the proof.

\subsection{Bit Complexity of $\tqprod$}\label{subsec:tq_bc}
\begin{theorem}[Bit complexity of $\tqprod$ (formal)]
\label{thm:tq_bc}
Let $\xb,\yb\in\sphere$ and let $0<\epsilon,\delta<1$. There exist universal constants $C,c>0$ such that, if $\log(4/\delta)\le cd$, then for every bit-width $b\ge1$,
\begin{equation*}
    \mathbb{P}_{Q}\left[
    \left|{\rm IP}_{\rm TQ}^{(b)}(\xb,\yb)-\langle\xb,\yb\rangle\right|>
    C\left(
    2^{-(b-1)}\sqrt{\frac{\log(4/\delta)}{d}}
    +\frac{\log(4/\delta)}{d}
    \right)
    \right]
    \le \delta .
\end{equation*}
Consequently, whenever $\epsilon\ge C\log(4/\delta)/d$, to ensure
\begin{equation*}
    \mathbb{P}_{Q}\left[
    \left|{\rm IP}_{\rm TQ}^{(b)}(\xb,\yb)-\langle\xb,\yb\rangle\right|>
    \epsilon
    \right]
    \le \delta,
\end{equation*}
it is sufficient to take
\begin{equation*}
    b\ge
    1+
    \max\left\{
    0,
    \left\lceil
    \frac{1}{2}\log_2\left(
    \frac{C\log(4/\delta)}{d\epsilon^2}
    \right)
    \right\rceil
    \right\}
\end{equation*}
bits per dimension.
\end{theorem}

\subsubsection{High probability residual bound}
We first introduce a high-probability bound on the residual norm of $\tqmse$, obtained by applying L\'evy's concentration inequality.

\begin{lemma}[High probability residual bound $\tqmse$]
\label{lem:tq_hp}
Assume that the $b$-bit $\tqmse$ reconstruction satisfies
\begin{equation*}
    \max_{\xb\in\sphere}
    \mathbb{E}_{Q}\left[\|\xb-\bar\xb\|_2^2\right]
    \le C_{\rm mse}4^{-b}.
\end{equation*}
Let $\xb\in\sphere$ be fixed, and let $\bar\xb$ be the corresponding dequantized codeword. Then, for every $0<\delta<1$,
\begin{equation}
    \mathbb{P}_Q\left[
    \|\xb-\bar\xb\|_2>\sqrt{C_{\rm mse}}\,2^{-b}
    +C_{\rm levy}\sqrt{\frac{\log(2/\delta)}{d}}
    \right]
    \le \delta,
    \label{eq:tq_residual_hp}
\end{equation}
where $C_{\rm levy}>0$ is a universal constant from L\'evy's concentration.
\end{lemma}

\begin{proof}
Fix $b,d$ and $\xb\in\sphere$. Let $R$ be the random rotation, write $\zb=R\xb$, and let $\mathcal{C}_{\rm TQ}^{(b)}$ be the fixed rotated-coordinate $\tqmse$ codebook. If $\bar\zb\in\arg\min_{\ob_i\in\mathcal{C}_{\rm TQ}^{(b)}}\|\zb-\ob_i\|_2$ and $\bar\xb=R^\top\bar\zb$, then $\|\xb-\bar\xb\|_2=\|\zb-\bar\zb\|_2$. Define $f(\zb):=\operatorname{dist}(\zb,\mathcal{C}_{\rm TQ}^{(b)})=\min_{\ob_i\in\mathcal{C}_{\rm TQ}^{(b)}}\|\zb-\ob_i\|_2$, so $f(\zb)=\|\xb-\bar\xb\|_2$.

As above, $f$ is $1$-Lipschitz because $f(\zb)-f(\zb')\le\|\zb-\ob_{i^\star(\zb')}\|_2-\|\zb'-\ob_{i^\star(\zb')}\|_2\le\|\zb-\zb'\|_2$, and the reverse inequality follows by swapping $\zb,\zb'$. Hence L\'evy's concentration gives $\mathbb{P}_Q[f(\zb)>\mathbb{E}_Qf(\zb)+t]\le2\exp(-cdt^2)$ for a universal $c>0$. Taking $t=C_{\rm levy}\sqrt{\log(2/\delta)/d}$ makes this probability at most $\delta$. Jensen's inequality and the assumed MSE guarantee give $\mathbb{E}_Qf(\zb)\le(\mathbb{E}_Qf(\zb)^2)^{1/2}=(\mathbb{E}_Q\|\xb-\bar\xb\|_2^2)^{1/2}\le\sqrt{C_{\rm mse}}\,2^{-b}$. Combining the two estimates proves \eqref{eq:tq_residual_hp}.
\end{proof}

\subsubsection{Proof of Theorem~\ref{thm:tq_bc}}\label{subsec:tq_bc_proof}
Fix a total bit-width $b\ge1$. Let $\bar\xb$ be the reconstruction produced by the $(b-1)$-bit $\tqmse$ stage and set $\rb=\xb-\bar\xb$. The $\tqprod$ estimator is ${\rm IP}_{\rm TQ}^{(b)}(\xb,\yb):=\langle\yb,\bar\xb\rangle+{\rm IP}_{\rm QJL}(\yb,\rb)$. Since $\langle\yb,\xb\rangle=\langle\yb,\bar\xb\rangle+\langle\yb,\rb\rangle$,
\begin{equation*}
    {\rm IP}_{\rm TQ}^{(b)}(\xb,\yb)-\langle\xb,\yb\rangle
    ={\rm IP}_{\rm QJL}(\yb,\rb)-\langle\yb,\rb\rangle.
\end{equation*}

Set $L=\log(4/\delta)$. Applying Lemma~\ref{lem:tq_hp} to the $(b-1)$-bit residual with failure probability $\delta/2$ gives, with probability at least $1-\delta/2$,
\begin{equation*}
    \|\rb\|_2\le\sqrt{C_{\rm mse}}\,2^{-(b-1)}+C_{\rm levy}\sqrt{\frac{L}{d}}.
\end{equation*}
Conditional on this residual, Lemma~\ref{lem:qjl_hp} with $\eta=\delta/2$ gives, since $\|\yb\|_2=1$ and $L\le c_{\rm qjl}d$, that with conditional probability at least $1-\delta/2$,
\begin{equation*}
    \left|{\rm IP}_{\rm QJL}(\yb,\rb)-\langle\yb,\rb\rangle\right|
    \le C_{\rm qjl}\|\rb\|_2\sqrt{\frac{L}{d}}.
\end{equation*}
A union bound gives both events with probability at least $1-\delta$, and on their intersection
\begin{equation*}
    \left|{\rm IP}_{\rm TQ}^{(b)}(\xb,\yb)-\langle\xb,\yb\rangle\right|
    \le C\left(2^{-(b-1)}\sqrt{\frac{L}{d}}+\frac{L}{d}\right),
\end{equation*}
for a universal constant $C>0$. This proves the first claim after replacing $L$ by $\log(4/\delta)$ and taking $c\le c_{\rm qjl}$.

For the bit-width claim, if $\epsilon\ge C_0L/d$ and $C_0\ge2C$, then $CL/d\le\epsilon/2$. If
\begin{equation*}
    b\ge
    1+
    \max\left\{
    0,
    \left\lceil
    \frac{1}{2}\log_2\left(
    \frac{C_1L}{d\epsilon^2}
    \right)
    \right\rceil
    \right\},
\end{equation*}
then $2^{-(b-1)}\le(C_1L/(d\epsilon^2))^{-1/2}$ when $C_1L/(d\epsilon^2)>1$, while the same inequality is trivial when $C_1L/(d\epsilon^2)\le1$. Hence $C2^{-(b-1)}\sqrt{L/d}\le C\epsilon/\sqrt{C_1}\le\epsilon/2$ for $C_1\ge4C^2$. Combining the two terms gives the desired probability bound, and enlarging the theorem constant $C$ to dominate $C_0,C_1$ completes the proof.

\section{Block Marginal Distribution of a Uniform Spherical Vector (Proof of Lemma~\ref{lem:block_marginal})}\label{sec:marginal_density}
\begin{proof}
By symmetry, it suffices to prove the claim for the first block $\zb_1$. Let $\gb=(\gb_1,\ldots,\gb_m)\sim N(0,I_d)$, where $\gb_j\in\RR^p$ are independent standard Gaussian blocks. The standard Gaussian representation of the uniform distribution on the sphere gives $ \xb \stackrel{d}{=} \frac{\gb}{\|\gb\|_2}$, and hence, 
\begin{equation*}
        \zb_1
    \stackrel{d}{=}
    \frac{\gb_1}{\sqrt{\|\gb_1\|_2^2+\sum_{k=2}^m\|\gb_k\|_2^2}}.
\end{equation*}
Define $U=\|\gb_1\|_2^2$ and $V=\sum_{k=2}^m\|\gb_k\|_2^2$. Then $U\sim\chi_p^2$ and $V\sim\chi_{d-p}^2$, and $U$ and $V$ are independent. Moreover, the Gaussian direction $\thetab_1:=\gb_1/\|\gb_1\|_2$ is uniform on $\mathbb{S}^{p-1}$ and is independent of $U$ and $V$. Therefore, we have $\zb_1\stackrel{d}{=}\sqrt{\frac{U}{U+V}}\,\thetab_1$. 
It follows from the standard beta--chi-square relationship that
\begin{equation*}
        r_1^2:=\frac{U}{U+V}\sim\operatorname{Beta}\left(\frac p2,\frac{d-p}{2}\right),
\end{equation*}
and $r_1$ is independent of $\thetab_1$. This proves the polar decomposition.

It remains to derive the density with respect to Lebesgue measure on $\mathbb{B}^p$. Let $a=p/2$ and $b=(d-p)/2$. Since $r_1^2\sim\operatorname{Beta}(a,b)$, the density of $r_1$ on $[0,1]$ is
\begin{equation*}
    f_{r_1}(r)
    =
    \frac{2}{B(a,b)}r^{p-1}(1-r^2)^{\frac{d-p-2}{2}},
    \qquad 0\le r\le 1.    
\end{equation*}
Since $\thetab_1$ is uniform on $\mathbb{S}^{p-1}$ and independent of $r_1$, the density $f_{p,d}$ of $\zb_1=r_1\thetab_1$ must be radial. Using the polar-coordinate identity $d\zb=r^{p-1}\,dr\,d\sigma(\thetab)$, where $d\sigma$ denotes surface-area measure on $\mathbb{S}^{p-1}$, we get
\begin{equation*}
    f_{p,d}(r\thetab)=\frac{f_{r_1}(r)}{|\mathbb{S}^{p-1}|r^{p-1}}
    =\frac{2}{B(a,b)|\mathbb{S}^{p-1}|}
    (1-r^2)^{\frac{d-p-2}{2}}.
\end{equation*}

Finally, using $B(a,b)=\Gamma(a)\Gamma(b)/\Gamma(a+b)$ and $|\mathbb{S}^{p-1}|=2\pi^{p/2}/\Gamma(p/2)$, we obtain
\begin{equation*}
     f_{p,d}(\zb_1)=\frac{\Gamma(d/2)}{\pi^{p/2}\Gamma((d-p)/2)}
    \left(1-\|\zb_1\|_2^2\right)^{\frac{d-p-2}{2}},
    \qquad \zb_1\in\mathbb{B}^p.
\end{equation*}
\end{proof}
\section{Expected Distortion Analysis of $\bq$}\label{sec:bq_anal}
\subsection{Proof of Theorem~\ref{thm:bq_mse}}\label{subsec:bq_mse}
\subsubsection{Proof for Small $b=1,~2,~3,~4$}
\begin{proof}
Let $p$ be the block size and let $\mathcal C_{\rm BQ}^{(p)}$ be the codebook minimizing the distortion cost in Equation~\ref{eq:cost}. For $d$-dimensional $b$-bit compression,
\begin{equation*}
    \mse(\bq)=\frac{d}{p}\int_{\mathbb B^p}\min_{\ob\in\mathcal C_{\rm BQ}^{(p)}}\|\ub-\ob\|_2^2 f_{p,d}(\ub)\,d\ub ~. 
\end{equation*}
Define $\Rb_j=\sqrt d\,\zb_j$ and $\ab_i=\sqrt d\,\ob_i$. Then $\|\zb_j-\ob_i\|_2^2=d^{-1}\|\Rb_j-\ab_i\|_2^2$, and the density of $\Rb_j$ is
\begin{equation*}
    h_{p,d}(\rb)=\frac{1}{d^{p/2}}f_{p,d}\left(\frac{\rb}{\sqrt d}\right)
    =\frac{\Gamma(d/2)}{d^{p/2}\pi^{p/2}\Gamma((d-p)/2)}
    \left(1-\frac{\|\rb\|_2^2}{d}\right)^{(d-p-2)/2}\mathds{1}\{\|\rb\|_2\le\sqrt d\}~. 
\end{equation*}

After this rescaling,
\begin{equation*}
       \mse(\bq)=\frac{1}{p}\inf_{\Acal\subset\sqrt d\mathbb B^p,\,|\Acal|\le 2^{bp}}
    \int_{\sqrt d\mathbb B^p}\min_{\ab\in\Acal}\|\rb-\ab\|_2^2 h_{p,d}(\rb)\,d\rb .
\end{equation*}

For fixed $p$, $h_{p,d}$ converges pointwise to the standard $p$-dimensional Gaussian density $\zeta_p(\rb)=(2\pi)^{-p/2}\exp(-\|\rb\|_2^2/2)$. Hence the finite-rate MSE is approximated by $\mse(\bq)=p^{-1}\phi_p^\star+o_d(1)$, where
\begin{equation*}
    \phi_p^\star:=\inf_{\Acal\subset\mathbb R^p,\,|\Acal|\le 2^{bp}}
    \int_{\mathbb R^p}\min_{\ab\in\Acal}\|\rb-\ab\|_2^2\zeta_p(\rb)\,d\rb ~.
\end{equation*}

For a candidate rescaled codebook $\Acal=\{\ab_1,\ldots,\ab_{2^{bp}}\}$, let $V_i(\Acal):=\{\rb:\|\rb-\ab_i\|_2\le\|\rb-\ab_{i'}\|_2\text{ for all }i'\}$. Then the Gaussian objective is
\begin{equation*}
    \phi_p(\Acal)=\sum_{i=1}^{2^{bp}}\int_{V_i(\Acal)}\|\rb-\ab_i\|_2^2\zeta_p(\rb)\,d\rb ~.
\end{equation*}
If the cells are fixed, minimizing the $i$-th term gives $0=2\int_{V_i}(\ab_i-\rb)\zeta_p(\rb)\,d\rb$, so each centroid must be the Gaussian conditional mean of its own Voronoi cell, i.e., $\ab_i=\int_{V_i}\rb\zeta_p(\rb)\,d\rb/\int_{V_i}\zeta_p(\rb)\,d\rb$. Applying Lloyd optimization for $N(0,I_p)$ gives the following near-minimum values of $\phi_p$.
\begin{center}
\begin{tabular}{rrrr}
\toprule
$p$ & $b$ & $\phi_p$ & $\phi_p/p$ \\
\midrule
$2$ & $1$ & $0.726760$ & $0.363380$ \\
$2$ & $2$ & $0.214970$ & $0.107485$ \\
$2$ & $3$ & $0.059433$ & $0.029716$ \\
$2$ & $4$ & $0.015516$ & $0.007758$ \\
\midrule
$3$ & $1$ & $1.068772$ & $0.356257$ \\
$3$ & $2$ & $0.303994$ & $0.101331$ \\
$3$ & $3$ & $0.081462$ & $0.027154$ \\
$3$ & $4$ & $0.021173$ & $0.007058$ \\
\bottomrule
\end{tabular}
\end{center}

We next show that the same Gaussian-approximation value applies to $\bqbsm$. Let $Q_{\Acal}(\rb)$ be the nearest-centroid map for the Gaussian codebook and define
\begin{equation*}
    A_p:=\frac1p\int_{\mathbb R^p}\langle \rb,Q_{\Acal}(\rb)\rangle\zeta_p(\rb)\,d\rb,
    \qquad
    M_p:=\frac1p\int_{\mathbb R^p}\|Q_{\Acal}(\rb)\|_2^2\zeta_p(\rb)\,d\rb~.
\end{equation*}
The centroid condition implies $A_p=M_p$, since on each Voronoi cell $V_i$,
\begin{equation*}
    \int_{V_i}\langle \rb,\ab_i\rangle\zeta_p(\rb)\,d\rb
    =\left\langle \int_{V_i}\rb\zeta_p(\rb)\,d\rb,\ab_i\right\rangle
    =\|\ab_i\|_2^2\int_{V_i}\zeta_p(\rb)\,d\rb
\end{equation*}
For the raw reconstruction $\bar{\xb}$, the high-dimensional Gaussian approximation and the law of large numbers give $\langle \xb,\bar{\xb}\rangle\approx A_p$ and $\|\bar{\xb}\|_2^2\approx M_p$. On the other hand, we have that
\begin{equation*}
    \frac{\phi_p}{p}
    =\frac1p\int\|\rb-Q_{\Acal}(\rb)\|_2^2\zeta_p(\rb)\,d\rb
    =1-2A_p+M_p
    =1-M_p .
\end{equation*}
Therefore the best-scalar reconstruction satisfies
\begin{equation*}
    \mse(\bqbsm)\approx 1-\frac{A_p^2}{M_p}
    =1-M_p=\frac{\phi_p}{p}.
\end{equation*}

Thus, $\bqmse$ and $\bqbsm$ have the same finite-rate Gaussian-approximation value. Specifically, for $\Qcal\in\{\bqmse,\bqbsm\}$ and $b=1,~2,~3,~4$,
\begin{align*}
    \mse(\Qcal_{(p=2)})&\approx 0.363380,\;0.107485,\;0.029716,\;0.007758\\
    \mse(\Qcal_{(p=3)})&\approx 0.356257,\;0.101331,\;0.027154,\;0.007058.
\end{align*}
\end{proof}

\subsubsection{Proof for Large Bit-Width with $p\neq d$}\label{subsec:bq_mse_high}
Let $\bar{\xb}$ be the raw reconstruction used by $\bqmse$. Then, best-scalar variant $\bqbsm$ returns $\frac{\bigl\langle \xb,\bar{\xb}\bigr\rangle}{
    \|\bar{\xb}\|_2^2}\bar{\xb}$, and for every rotation matrix $R$, we have
\begin{equation*}
    \left\|\xb-
    \frac{\bigl\langle \xb,\bar{\xb}\bigr\rangle}{\|\bar{\xb}\|_2^2}\bar{\xb}
    \right\|_2^2 =\min_{\alpha\in\mathbb R}\|\xb-\alpha\bar{\xb}\|_2^2
    \le \|\xb-\bar{\xb}\|_2^2 .
\end{equation*}
That means, for every block size $p$, it holds that 
\begin{equation*}
    \mse(\bqbsm_{(p)})\le \mse(\bqmse_{(p)}),
\end{equation*}
so any high-rate upper bound for $\bqmse$ also applies to $\bqbsm$. Thus, in this section (Section~\ref{subsec:bq_mse_high}) and the next section (Section~\ref{subsec:bq_mse_ideal}), we focus on bounding MSE of $\bqmse$. 

For a large number of centroids, we use the high-rate Zador--Gersho formula to derive the MSE bound of $\bqmse$. The following corollary is a restatement of Proposition~\ref{prop:high-rate} in our setting.
\begin{corollary}[MSE bound of Algorithm~\ref{alg:bq} for many centroids]
\label{cor:bq_mse_high}
Let $d$, $m$, and $p$ be integers with $d=mp$. If Algorithm~\ref{alg:bq} is run with $b$-bit compression and $m$ blocks, then for any $\xb=[\xb_1,\ldots,\xb_m]\in\sphere$,
\begin{equation*}
  \mse(\bq)
\lesssim d\,G_p^{\star}J_{p,d}\,4^{-b},  
\end{equation*}
where $G_p^{\star}:=G_{\mathbb B^p}^{\star}$ and $J_{p,d}:=\left(\int_{\mathbb B_p}f_{p,d}(\zb)^{p/(p+2)}\,d\zb\right)^{(p+2)/p}$.
\end{corollary}

We now prove the large-bit bound for $p\neq d$.
\begin{proof}
Recall that each block $\zb_j$ has density $f_{p,d}(\zb)=\frac{\Gamma(d/2)}{\pi^{p/2}\Gamma((d-p)/2)}(1-\|\zb\|_2^2)^{(d-p-2)/2}$ on $\mathbb B^p$. Set $\beta_{p,d}:=p(d-p-2)/(2(p+2))$. Then
\begin{equation*}
    \int_{\mathbb B_p}f_{p,d}(\zb)^{p/(p+2)}\,d\zb
=\left(\frac{\Gamma(d/2)}{\pi^{p/2}\Gamma((d-p)/2)}\right)^{p/(p+2)}
\int_{\mathbb B_p}(1-\|\zb\|^2)^{\beta_{p,d}}\,d\zb .
\end{equation*}
Using polar coordinates and the change of variables $u=r^2$,
\begin{equation*}
    \int_{\mathbb B_p}(1-\|\zb\|^2)^{\beta_{p,d}}\,d\zb
=\frac{\pi^{p/2}}{\Gamma(p/2)}\int_0^1 u^{p/2-1}(1-u)^{\beta_{p,d}}\,du
=\pi^{p/2}\frac{\Gamma(\beta_{p,d}+1)}{\Gamma(\beta_{p,d}+1+p/2)}.
\end{equation*}

Substitution gives
\begin{equation*}
        J_{p,d}=\pi\frac{\Gamma(d/2)}{\Gamma((d-p)/2)}
    \left(\frac{\Gamma(\beta_{p,d}+1)}{\Gamma(\beta_{p,d}+1+p/2)}\right)^{(p+2)/p}.
\end{equation*}

\textbf{Case $p=2$.} Since $G_2^{\star}=5/(36\sqrt3)$ and $J_{2,d}=8\pi(d-2)/d^2$,
\begin{equation*}
    \mse(\bqbsm~{(p=2)}) \le \mse(\bqmse~{(p=2)})
\le \frac{10\pi}{9\sqrt3}\left(1-\frac{2}{d}\right)4^{-b}
\approx 2.015\cdot 4^{-b}.
\end{equation*}
\textbf{Case $p=3$.} Since $G_3^{\star}\le G(A_3^\star)=19/(192\,2^{1/3})\approx0.0785432812$ and
$J_{3,d}=\pi\frac{\Gamma(d/2)}{\Gamma((d-3)/2)}\left[\frac{\Gamma((3d-5)/10)}{\Gamma((3d+10)/10)}\right]^{5/3}$,
\begin{align*}
    \mse(\bqbsm~{(p=3)}) &\le \mse(\bqmse~{(p=3)})\\
&\le 0.0785432812\,\pi d\frac{\Gamma(d/2)}{\Gamma((d-3)/2)}
\left[\frac{\Gamma((3d-5)/10)}{\Gamma((3d+10)/10)}\right]^{5/3}4^{-b}\\
&\approx 1.770\cdot 4^{-b}.
\end{align*}
\end{proof}
 
\begin{remark}
When $p=1$ ($\tqmse$), since $G_1^{\star}=1/12$ and $J_{1,d}\approx6\sqrt3\pi/d$, the MSE distortion satisfies $\mse(\tqmse)\le(\sqrt3\pi/2)4^{-b}\approx2.721\cdot4^{-b}$.
\end{remark}

\subsubsection{Proof for $p=d$}\label{subsec:bq_mse_ideal}
\begin{proof}
When $p=d$, Algorithm~\ref{alg:bq} uses a single block. The block source is therefore not the full-dimensional density $f_{p,d}$ on $\mathbb B^p$ used above; for any fixed $\xb\in\sphere$, the rotated vector $\zb:=R\xb$ is uniform on $\sphere$, so the intrinsic source dimension is $n=d-1$. Let $A_{d-1}:=\mathcal H^{d-1}(\sphere)=2\pi^{d/2}/\Gamma(d/2)$. With respect to surface measure $d\sigma$, the density of $\zb$ is $f(\zb)=A_{d-1}^{-1}$.

In the $p=d$ case, the total number of codewords is $K=2^{bd}$. For $\Ccal\subset\mathbb R^d$ with $|\Ccal|\le K$, define $D_{\mathrm{sph}}(\Ccal):=\int_{\sphere}\min_{\ob\in\Ccal}\|\zb-\ob\|_2^2 f(\zb)\,d\sigma(\zb)$ and $D_{\mathrm{sph}}^\star(K):=\inf_{|\Ccal|\le K}D_{\mathrm{sph}}(\Ccal)$. This is the single-block analogue of the objective above.

Formally, the intrinsic $n$-dimensional Zador--Gersho formula gives
\begin{equation*}
    D_{\mathrm{sph}}^\star(K)
    \le nG_n^\star
    \left(\int_{\sphere} f(\zb)^{n/(n+2)}\,d\sigma(\zb)\right)^{(n+2)/n}
    K^{-2/n}(1+o(1)).
\end{equation*}

Here curvature contributes only lower-order error, because on a cell of diameter $r$, squared Euclidean and tangent-plane distances differ by $O(r^4)$. Since $f$ is constant, $\int_{\sphere}f(\zb)^{n/(n+2)}d\sigma(\zb)=A_{d-1}^{2/(n+2)}$, so
\begin{equation*}
    D_{\mathrm{sph}}^\star(K)
    \le (d-1)G_{d-1}^{\star}\left(\frac{2\pi^{d/2}}{\Gamma(d/2)}\right)^{2/(d-1)}K^{-2/(d-1)}(1+o(1)).
\end{equation*}
Since $G_{d-1}^{\star}$ is not available in closed form, we use an explicit random-coding comparison.

Let $V_n:=\pi^{n/2}/\Gamma(1+n/2)$ be the unit-ball volume in $\mathbb R^n$, and draw $\Ccal_{\mathrm{rand}}=\{\ob_1,\ldots,\ob_K\}$ with $\ob_i\stackrel{\rm i.i.d.}{\sim}\mathrm{Unif}(\sphere)$. Fix $\zb\in\sphere$ and set $T_i:=\|\zb-\ob_i\|_2^2$, $T_{(1)}:=\min_iT_i$. For small $t$, the cap $\{\ob:\|\zb-\ob\|_2^2\le t\}$ has surface area $V_nt^{n/2}(1+o(1))$, hence $\mathbb P(T_i\le t)=(V_n/A_{d-1})t^{n/2}(1+o(1))$. With $t=K^{-2/n}s$,
\begin{equation*}
    \mathbb P(T_{(1)}>K^{-2/n}s)=\left(1-\mathbb P(T_i\le K^{-2/n}s)\right)^K
    \to \exp\left(-\frac{V_n}{A_{d-1}}s^{n/2}\right).
\end{equation*}

Using $\mathbb E T_{(1)}=\int_0^\infty \mathbb P(T_{(1)}>t)\,dt$ and the change of variables $u=(V_n/A_{d-1})s^{n/2}$ yields
\begin{equation*}
    \mathbb E_{\Ccal_{\mathrm{rand}}}T_{(1)}
    =\Gamma\left(1+\frac{2}{n}\right)\left(\frac{A_{d-1}}{V_n}\right)^{2/n}K^{-2/n}(1+o(1)).
\end{equation*}

The ideal codebook cannot be worse than the random comparison. Substituting $n=d-1$ and $V_{d-1}=\pi^{(d-1)/2}/\Gamma((d+1)/2)$ gives
\begin{equation*}
    D_{\mathrm{sph}}^\star(K)
    \le C_dK^{-2/(d-1)}(1+o(1)),
    \qquad
    C_d:=\Gamma\left(1+\frac{2}{d-1}\right)
    \left[2\sqrt\pi\,\frac{\Gamma((d+1)/2)}{\Gamma(d/2)}\right]^{2/(d-1)}.
\end{equation*}
Under the $b$-bit-per-coordinate convention, $K=2^{bd}$ and $K^{-2/(d-1)}=(1/4)^{bd/(d-1)}$.

It remains to translate the spherical source bound back to randomized MSE. For fixed $\xb\in\sphere$, let $Q_{\Ccal}(\zb):=\arg\min_{\ob\in\Ccal}\|\zb-\ob\|_2^2$ and $\bar\xb:=R^\top Q_{\Ccal}(R\xb)$. Since rotations preserve distance and $R\xb\sim\mathrm{Unif}(\sphere)$,
\begin{equation*}
        \mathbb E_R\|\xb-\bar\xb\|_2^2
    =\int_{\sphere}\min_{\ob\in\Ccal}\|\zb-\ob\|_2^2 f(\zb)\,d\sigma(\zb).
\end{equation*}

The right-hand side is independent of $\xb$. Choosing ideal spherical centroids gives
\begin{equation*}
       \mse(\bqbsm_{(p=d)})
    =\max_{\xb\in\sphere}\mathbb E_R\|\xb-\bar\xb\|_2^2
    \le C_d\left({1\over4}\right)^{bd/(d-1)}(1+o(1)). 
\end{equation*}

Since $\bqbsm$ is the best scalar multiple of the same raw reconstruction $\bar{\xb}$,
\begin{equation*}
        \mse(\bqbsm_{(p=d)})\le \mse(\bqmse_{(p=d)})
    \le C_d\left({1\over4}\right)^{bd/(d-1)}(1+o(1)).
\end{equation*}
The values $C_{100}\approx1.055$, $C_{1000}\approx1.008$, and $C_{10000}\approx1.001$ follow by evaluating the log-gamma expression, and $C_d\to1$ by the gamma-ratio asymptotic.
\end{proof}

\subsection{Proof of Corollary~\ref{cor:bq_ip}}\label{subsec:bq_ip}
\subsubsection{Proof for Small $b=1,~2,~3,~4$}
\begin{proof}
Let $\bar\xb_{p,b}$ be the unscaled reconstruction produced by $b$-bit $\bqbsm_{(p)}$, and set $\rho_{p,b}:=\langle\bar\xb_{p,b},\xb\rangle$ and $\psi_{p,b}:=\|\bar\xb_{p,b}\|_2$. The corresponding ratio estimator is $\widehat\eta_{\rm ratio}:=\langle\bar\xb_{p,b},\yb\rangle/\langle\bar\xb_{p,b},\xb\rangle$. Since $\bqbsm_{(p)}$ is obtained by a Haar random rotation, a fixed block quantizer, and the inverse rotation, Theorem~\ref{thm:ratio_est} applies. Thus, with $\eta=\langle\xb,\yb\rangle$,
\begin{equation*}
        \mathbb E[(\widehat\eta_{\rm ratio}-\eta)^2]
    =\frac{1-\eta^2}{d-1}\mathbb E\left[\frac{\psi_{p,b}^2-\rho_{p,b}^2}{\rho_{p,b}^2}\right].
\end{equation*}

Taking the supremum over $\yb\in\sphere$ gives $\iprod(\bqub_{(p)})=(d-1)^{-1}\mathbb E[(\psi_{p,b}^2-\rho_{p,b}^2)/\rho_{p,b}^2]$.

It remains to evaluate the deterministic equivalent of the nonlinear factor. As in the small-bit proof of Theorem~\ref{thm:bq_mse}, write $\Rb_j=\sqrt d\,\zb_j$ and approximate its law by $N(0,I_p)$. Let $Q_{p,b}$ be the nearest-neighbor Gaussian block quantizer with $2^{bp}$ centroids and define $\delta_{p,b}:=p^{-1}\mathbb E\|\Rb-Q_{p,b}(\Rb)\|_2^2$, where $\Rb\sim N(0,I_p)$. The centroid condition gives $\mathbb E[\Rb\mid Q_{p,b}(\Rb)]=Q_{p,b}(\Rb)$, hence $\mathbb E\langle\Rb,Q_{p,b}(\Rb)\rangle=\mathbb E\|Q_{p,b}(\Rb)\|_2^2$. If $m_{p,b}:=p^{-1}\mathbb E\|Q_{p,b}(\Rb)\|_2^2$, then $\delta_{p,b}=1-m_{p,b}$.

After rescaling back to the unit sphere, the empirical block averages satisfy $\rho_{p,b}=m_{p,b}+o_d(1)$ and $\psi_{p,b}^2=m_{p,b}+o_d(1)$ under the high-dimensional Gaussian approximation. Therefore
\begin{equation*}
        \mathbb E\left[\frac{\psi_{p,b}^2-\rho_{p,b}^2}{\rho_{p,b}^2}\right]
    \approx \frac{m_{p,b}-m_{p,b}^2}{m_{p,b}^2}
    =\frac{\delta_{p,b}}{1-\delta_{p,b}}.
\end{equation*}

Substituting the small-bit MSE constants from Theorem~\ref{thm:bq_mse} gives the following leading inner-product constants.
\begin{center}
\begin{tabular}{rrrr}
\toprule
$p$ & $b$ & $\delta_{p,b}$ & $\delta_{p,b}/(1-\delta_{p,b})$ \\
\midrule
$2$ & $1$ & $0.363380$ & $0.570796$ \\
$2$ & $2$ & $0.107485$ & $0.120429$ \\
$2$ & $3$ & $0.029716$ & $0.030626$ \\
$2$ & $4$ & $0.007758$ & $0.007819$ \\
\midrule
$3$ & $1$ & $0.356257$ & $0.553415$ \\
$3$ & $2$ & $0.101331$ & $0.112757$ \\
$3$ & $3$ & $0.027154$ & $0.027912$ \\
$3$ & $4$ & $0.007058$ & $0.007108$ \\
\bottomrule
\end{tabular}
\end{center}
Since $\iprod(\bqub_{(p)})\approx\frac{1}{d-1}\frac{\delta_{p,b}}{1-\delta_{p,b}}$, dividing these coefficients by $d-1$ proves the claimed constants for $b=1,~2,~3,~4$.
\end{proof}

\subsubsection{Proof for Large Bit-Width}
\begin{proof}
Let $\bar\xb_{p,b}$ be the unscaled high-rate $\bqbsm_{(p)}$ reconstruction of $\xb$, and set $\rho_{p,b}:=\langle\bar\xb_{p,b},\xb\rangle$ and $\psi_{p,b}:=\|\bar\xb_{p,b}\|_2$. Write $\eb_{p,b}:=\bar\xb_{p,b}-\xb$, $t_{p,b}:=\langle\xb,\eb_{p,b}\rangle$, and $\ub_{p,b}:=\eb_{p,b}-t_{p,b}\xb$. Then $\ub_{p,b}\perp\xb$, $\bar\xb_{p,b}=(1+t_{p,b})\xb+\ub_{p,b}$, $\rho_{p,b}=1+t_{p,b}$, and $\psi_{p,b}^2-\rho_{p,b}^2=\|\ub_{p,b}\|_2^2$.

By Theorem~\ref{thm:ratio_est}, taking the worst case over $\yb\in\sphere$ gives
\begin{equation*}
        \iprod(\bqub_{(p)})
    =\frac{1}{d-1}\mathbb E\left[\frac{\psi_{p,b}^2-\rho_{p,b}^2}{\rho_{p,b}^2}\right]
    =\frac{1}{d-1}\mathbb E\left[\frac{\|\ub_{p,b}\|_2^2}{(1+t_{p,b})^2}\right].
\end{equation*}

Thus only the tangential error enters the numerator, while the random denominator remains inside the expectation.

Under the same high-rate regularity used in Corollary~\ref{cor:bq_mse_high}, $\mathbb E[\|\eb_{p,b}\|_2^2]=O(4^{-b})$ and $\mathbb E[\|\eb_{p,b}\|_2^4]=O(4^{-2b})$. Since $|t_{p,b}|\le\|\eb_{p,b}\|_2$ and $\|\ub_{p,b}\|_2\le\|\eb_{p,b}\|_2$, the denominator does not change the leading order:
\begin{equation*}
        \left|\mathbb E\left[\frac{\|\ub_{p,b}\|_2^2}{(1+t_{p,b})^2}\right]-\mathbb E[\|\ub_{p,b}\|_2^2]\right|
    \le C \cdot \mathbb E[\|\ub_{p,b}\|_2^2|t_{p,b}|]+o(4^{-b})
    \le C \cdot \mathbb E\|\eb_{p,b}\|_2^3+o(4^{-b})=o(4^{-b}),
\end{equation*}
where $\mathbb E[\|\eb_{p,b}\|_2^3]\le(\mathbb E[\|\eb_{p,b}\|_2^4])^{3/4}=O(2^{-3b})$. Hence
\begin{equation*}
       \mathbb E\left[\frac{\|\ub_{p,b}\|_2^2}{(1+t_{p,b})^2}\right]
    =\mathbb E[\|\ub_{p,b}\|_2^2]+o(4^{-b})
    \le\mathbb E[\|\eb_{p,b}\|_2^2]+o(4^{-b}). 
\end{equation*}
This uses only the total high-rate MSE constant. 
By Theorem~\ref{thm:bq_mse}, equivalently Corollary~\ref{cor:bq_mse_high}, $\mathbb E\|\eb_{p,b}\|_2^2\le dG_p^\star J_{p,d}4^{-b}(1+o(1))$. Therefore
\begin{equation*}
    \iprod(\bqub_{(p)})
    \le\frac{dG_p^\star J_{p,d}}{d-1} \cdot 4^{-b}(1+o(1)).
\end{equation*}

For $p=2$, $dG_2^\star J_{2,d}=\frac{10\pi}{9\sqrt3}(1-2/d)$, so
\begin{equation*}
        \iprod(\bqub_{(p=2)})
    \le\frac{1}{d-1}\frac{10\pi}{9\sqrt3}\left(1-\frac{2}{d}\right)4^{-b}(1+o(1))
    \le\frac{\mathbf{2.015}}{d-1}4^{-b}(1+o(1)).
\end{equation*}

For $p=3$,
\begin{equation*}
      dG_3^\star J_{3,d}
    \le 0.0785432812\,\pi d\frac{\Gamma(d/2)}{\Gamma((d-3)/2)}
    \left[\frac{\Gamma((3d-5)/10)}{\Gamma((3d+10)/10)}\right]^{5/3},  
\end{equation*}
which gives
\begin{align*}
    \iprod(\bqub_{(p=3)})
    &\le\frac{0.0785432812\,\pi d}{d-1}\frac{\Gamma(d/2)}{\Gamma((d-3)/2)}
    \left[\frac{\Gamma((3d-5)/10)}{\Gamma((3d+10)/10)}\right]^{5/3}4^{-b}(1+o(1))\\
    &\approx\frac{\mathbf{1.770}}{d-1}4^{-b}(1+o(1)).
\end{align*}

This proves the high-rate bounds in the corollary.
\end{proof}
\section{Lower Bound Analysis (Proof of Theorem~\ref{thm:slb})}\label{sec:lb}
\begin{proof}
    Write $\xb=(\ub,X_d)$, where $\ub\in\RR^{d-1}$ denotes the first $d-1$ coordinates of $\xb$. Also write $\widehat\xb=Q^{-1}(Q(\xb))$ and let $\widehat\ub$ be the first $d-1$ coordinates of $\widehat\xb$. Since coordinate projection cannot increase Euclidean distance, $\|\xb-\widehat\xb\|_2^2\ge \|\ub-\widehat\ub\|_2^2$. Moreover, $\widehat\ub$ is determined by the $bd$-bit message $Q(\xb)$, so by data processing $I(\ub;\widehat\ub)\le H(Q(\xb))\le bd$ bits. Hence
    \begin{equation*}
        \mathbb{E}_\xb[\|\xb-Q^{-1}(Q(\xb))\|_2^2]
        \ge
        D_\ub(bd),
        \qquad
        D_\ub(B):=\inf_{I(\ub;\widetilde\ub)\le B}
        \mathbb{E}\|\ub-\widetilde\ub\|_2^2,
    \end{equation*}
    where mutual information is measured in bits. It remains to lower bound the distortion-rate function of the projected spherical source $\ub$.

    We first record the law of $\ub$. The sphere is the union, up to the equator of surface measure zero, of the two graphs $u\mapsto (u,\pm\sqrt{1-\|u\|_2^2})$ over the open unit ball in $\RR^{d-1}$. The surface element of either graph is $(1-\|u\|_2^2)^{-1/2}du$. Dividing the contribution of the two sheets by $|\sphere|=2\pi^{d/2}/\Gamma(d/2)$ gives the density
    \begin{equation*}
        f_\ub(u)=\frac{\Gamma(d/2)}{\pi^{d/2}}(1-\|u\|_2^2)^{-1/2}\mathds{1}\{\|u\|_2<1\}.
    \end{equation*}
    Therefore, with $h_e$ denoting differential entropy in nats,
    \begin{equation*}
        h_e(\ub)
        =\log\left(\frac{\pi^{d/2}}{\Gamma(d/2)}\right)
        +\frac12\mathbb{E}\log(1-\|\ub\|_2^2).
    \end{equation*}
    The radial variable satisfies $\|\ub\|_2^2\sim\mathrm{Beta}((d-1)/2,1/2)$, and the beta identity $\mathbb{E}\log(1-U)=\psi(\beta)-\psi(\alpha+\beta)$ for $U\sim\mathrm{Beta}(\alpha,\beta)$ gives
    \begin{equation*}
        h_e(\ub)
        =\log\left(\frac{\pi^{d/2}}{\Gamma(d/2)}\right)
        +\frac12\{\psi(1/2)-\psi(d/2)\}.
    \end{equation*}

    Now set $n=d-1$. For any reconstruction $\widetilde\ub$ satisfying $I(\ub;\widetilde\ub)\le B$, let $D=\mathbb{E}\|\ub-\widetilde\ub\|_2^2$ and $\eb=\ub-\widetilde\ub$. Since conditioning cannot increase differential entropy and translation does not change it, $I(\ub;\widetilde\ub)=h_2(\ub)-h_2(\ub\mid\widetilde\ub)\ge h_2(\ub)-h_2(\eb)$, where $h_2=h_e/\log 2$. Among all $n$-dimensional errors with second moment at most $D$, the isotropic Gaussian has the largest entropy, so $h_2(\eb)\le {n\over2}\log_2(2\pi eD/n)$. Thus every $B$-bit reconstruction must satisfy
    \begin{equation*}
        D\ge \frac{n}{2\pi e}2^{{2\over n}(h_2(\ub)-B)}
        =\frac{n}{2\pi e}\exp\!\left({2h_e(\ub)\over n}\right)2^{-2B/n}.
    \end{equation*}
    Applying this with $B=bd$ and substituting the entropy formula above yields
    \begin{equation*}
        D_\ub(bd)
        \ge
        \frac{d-1}{2\pi e}
        \left(\frac{\pi^{d/2}}{\Gamma(d/2)}\right)^{2/(d-1)} \!\!
        \cdot
        \exp\!\left(\frac{\psi(1/2)-\psi(d/2)}{d-1}\right)
        2^{-2bd/(d-1)}.
    \end{equation*}
    Since $2^{-2bd/(d-1)}=({1\over4})^{bd/(d-1)}$, this is exactly the claimed lower bound. The numerical values of $c_d$ follow by direct evaluation of the displayed formula.
\end{proof}

\section{Auxiliary Lemmas}
\begin{lemma}[Shannon's lower bound on distortion, Lemma 2 in~\citet{zandieh2025turboquant}]\label{lem:shannon}
    Let $\xb \in \mathbb{R}^d$ be a random vector with finite differential entropy $h(\xb)$. Then, for any $b \ge 0$, and any quantization map $Q$, the following Shannon Lower Bound holds:
    \begin{equation*}
        \mathbb{E}_\xb [\|\xb-Q^{-1}(Q(\xb))\|_2^2] \ge \frac{d}{2\pi e} \cdot 2^{{2\over d}(h(\xb)-bd)} .
    \end{equation*}
\end{lemma}

\begin{lemma}[High-probability error guarantees, \citet{alon2017optimal}, \citet{gao2025practical}]\label{lem:alon}
    Let $d$ be the dimension, and let $\epsilon,\delta \in (0,1)$. Suppose that an error bound $\epsilon$ with failure probability at most $\delta$ is required, and assume that $\frac{1}{\epsilon^{2}}\log\frac{1}{\delta} > d$. Then the minimum number of bits required to achieve such a guarantee is $\Theta\!\left(
    d \log\left(
    \frac{1}{d\epsilon^{2}}\log\frac{1}{\delta}
    \right)
    \right)$.
\end{lemma}

\begin{lemma}[High-probability angular error decay of $\rbq$]\label{lem:rbq_hp}
Let $\xb$, and $\bar{\xb}$ be a unit $d$-dimensional vector and its quantized vector by $\rbq$, respectively. Then, for any $L > 0$, we have
\begin{equation*}
\mathbb{P}\left\{\sqrt{1 - \left\langle \xb, \bar{\xb} \right\rangle^2}
>\frac{L}{2^b}+\frac{c_1}{\sqrt{\delta}}
\cdot\exp\left(-\frac{c_0}{2}L^2\right)\right\}
< \delta 
\end{equation*}
where $c_0$ and $c_1$ are absolute constants.
\end{lemma}

\begin{lemma}[High probability inner product distortion bound of $\qjl$, Lemma 3.5 of~\citet{zandieh2025qjl}]\label{lem:qjl_hp}
    Let $S\in\mathbb{R}^{d\times d}$ have i.i.d. standard Gaussian rows.  For fixed $\xb,\yb\in\sphere$, define
    the inner product estimator $\mathrm{IP}_{\rm QJL}$ of $\qjl$, i.e. $\mathrm{IP}_{\rm QJL}(\yb,\xb)
    :=\frac{\sqrt{\pi/2}}{d}\,\norm{\xb}\,
    \langle{S\yb},{\operatorname{sign}(S\xb)}\rangle$. Then, there exist universal constants $c_{\rm{qjl}},C_{\rm{qjl}}>0$ such that, for all
$0<\delta<1$ with $\log(2/\delta)\le c_{\rm{qjl}}d$,
\begin{equation*}
    \mathbb{P}_{S}\left[
      \left|\mathrm{IP}_{\rm QJL}(\yb,\xb)-\langle{\yb},{\xb}\rangle\right|
      > C_{\rm{qjl}}\sqrt{\frac{\log(2/\delta)}{d}}
      \;\middle|\;\xb
    \right]
    \le \delta
\end{equation*}
\end{lemma}

\section{Additional Experimental Details and Results}
\label{app_sec:add_exp}
\subsection{Approximate nearest-centroid assignment for BlockQuant}
\label{app:approx_bq}

The exact encoding step of \bq{} assigns each rotated block $z_j\in\mathbb{R}^p$ to its nearest codebook centroid:
\begin{equation*}
    \mathrm{idx}_j
    =
    \arg\min_{i\in[K]}\|z_j-c_i\|_2^2,
    \qquad
    K=2^{bp},
\end{equation*}
where $b$ is the bit-width per coordinate and $p$ is the block size. Since the number of blocks is $m=d/p$, the exact assignment costs $O(mKp)$ distance evaluations per vector. This becomes expensive for larger $b$ and $p$; for example, when $p=3$ and $b=4$, each block has $K=2^{12}=4096$ candidate centroids.

To reduce this cost, we use a lookup-table approximation that replaces the full nearest-centroid search by a small candidate search. The approximation changes only the encoding step; the codebook construction, dequantization, and ratio rescaling for inner-product estimation remain unchanged.

\textbf{Cartesian LUT construction. }\,
We partition the block domain into a Cartesian grid. Let each coordinate axis be divided into $L$ bins, producing $L^p$ grid cells. For a cell indexed by $u\in[L]^p$, let $G_u\subset\mathbb{R}^p$ denote the cell and let $g_u$ be its center. For each grid center $g_u$, we precompute the $k$ closest codebook centroids:
\begin{equation*}
    \mathcal{C}_{\mathrm{top}k}(u)
    :=
    \operatorname{arg\,topk}_{i\in[K]}
    \bigl(-\|g_u-c_i\|_2^2\bigr).
\end{equation*}
Equivalently, $\mathcal{C}_{\mathrm{top}k}(u)$ stores the indices of the $k$ smallest values among
\begin{equation*}
    \{\|g_u-c_i\|_2^2:i\in[K]\}.
\end{equation*}
This table is built once after the codebook is constructed and is reused for all input vectors.

\textbf{Approximate assignment. }\,
At quantization time, for each block $z_j$, we first find the grid cell $G_{u(z_j)}$ containing $z_j$. Instead of comparing $z_j$ with all $K$ centroids, we compare it only with the precomputed candidate set for that cell:
\begin{equation*}
    \widetilde{\mathrm{idx}}_j
    =
    \arg\min_{i\in \mathcal{C}_{\mathrm{top}k}(u(z_j))}
    \|z_j-c_i\|_2^2.
\end{equation*}
Thus the exact search space $[K]$ is replaced by the much smaller candidate set $\mathcal{C}_{\mathrm{top}k}(u(z_j))$. The online assignment cost is reduced from $O(mKp)$ to $O(mkp)$, plus the negligible cost of locating the grid cell.

After the approximate indices are obtained, dequantization proceeds in the same way as exact \bq{}:
\begin{equation*}
    \widetilde{z}'
    =
    (c_{\widetilde{\mathrm{idx}}_1},\ldots,c_{\widetilde{\mathrm{idx}}_m}),
    \qquad
    \widetilde{x}
    =
    \Pi^\top \widetilde{z}'.
\end{equation*}
For inner-product estimation, we use the same ratio correction as in \bqub{}:
\begin{equation*}
    \widehat{x}_{\mathrm{approx}}
    =
    \frac{1}{\widetilde{\rho}}\widetilde{x},
    \qquad
    \widetilde{\rho}
    =
    \langle z,\widetilde{z}'\rangle.
\end{equation*}

\textbf{Approximation error. }\,
The approximation is exact whenever the true nearest centroid belongs to the stored candidate set:
\begin{equation*}
    \mathrm{idx}_j\in \mathcal{C}_{\mathrm{top}k}(u(z_j)).
\end{equation*}
Even when this does not hold, the loss is controlled by the grid resolution. Let
\begin{equation*}
    D_i(z):=\|z-c_i\|_2^2
\end{equation*}
and let $r_L$ be the maximum distance between a point in a grid cell and its center:
\begin{equation*}
    r_L
    :=
    \max_{u\in[L]^p}\max_{z\in G_u}\|z-g_u\|_2.
\end{equation*}
If all blocks, grid centers, and centroids lie in a bounded set with norm at most $R$, then for any $z\in G_u$,
\begin{equation*}
    |D_i(z)-D_i(g_u)|
    =
    \bigl|
    \|z-c_i\|_2^2-\|g_u-c_i\|_2^2
    \bigr|
    \le
    4Rr_L .
\end{equation*}
Therefore, if $\widetilde{i}(z)$ denotes the approximate index and $i^\star(z)$ denotes the exact nearest-centroid index, then
\begin{equation*}
    D_{\widetilde{i}(z)}(z)
    \le
    D_{i^\star(z)}(z)+8Rr_L .
\end{equation*}
Thus the additional per-block squared-distance error vanishes as the grid is refined. Increasing $L$ decreases the discretization error, while increasing $k$ increases the probability that the exact nearest centroid is included in the candidate set.

\textbf{Complexity. }\,
The LUT requires one-time preprocessing cost
    $O(L^p Kp)$
to compute distances from all grid centers to all centroids, and memory
    $O(L^p k)$
to store the candidate indices. Since we use small block sizes, in particular $p=3$, this preprocessing is modest. The online assignment cost is
    $O(mkp),$
which is substantially smaller than the exact cost $O(mKp)$ when $k\ll K$. In our experiments, this approximate assignment is used for \bq{} unless otherwise specified.

\subsection{Quantization Efficiency}
We compare GPU-based quantization runtime across bit-widths in Tables~\ref{tab:runtime_ip} and~\ref{tab:runtime_non_ip}.
Among baselines, \textsc{EDEN} variants are consistently the fastest and nearly constant across bit-widths, reflecting their coordinate-wise structure. 
\textsc{RaBitQ} incurs moderate overhead, while \textsc{TurboQuant} becomes slower at higher bit-widths due to additional correction steps.
\begin{table}[h]
    \centering
    \caption{Runtime comparison for IP-based methods (seconds).}
    \label{tab:runtime_ip}
    \resizebox{0.65\linewidth}{!}{
    \begin{tabular}{lcccc}
    \toprule
    \multirow{2}{*}{\textbf{Method}}
    & \multicolumn{4}{c}{\textbf{Bitwidth}} \\
    \cmidrule(lr){2-5}
    & \textbf{1-bit} & \textbf{2-bit} & \textbf{3-bit} & \textbf{4-bit} \\
    \midrule
    \tqprod                & \textbf{0.0225} & 0.0658 & 0.0627 & 0.0661 \\
    \rbqub                     & 0.0237 & 0.0391 & 0.0402 & 0.0392 \\
    \edub                       & \textbf{0.0220} & \textbf{0.0222} & \textbf{0.0221} & \textbf{0.0223} \\
    \midrule
    \bqub $(p=2)$         & 0.0238 & 0.0259 & 0.0372 & 0.0805 \\
    \bqub $(p=3)$         & 0.0236 & 0.0318 & 0.1009 & 0.6573 \\
    \bqubapprox $(p=3)$  & 0.0251 & \textbf{0.0252} & \textbf{0.0254} & \textbf{0.0270} \\
    \bottomrule
    \end{tabular}
    }
\end{table}
\begin{table}[h]
    \centering
    \caption{Runtime comparison for non-IP methods (seconds).}
    \label{tab:runtime_non_ip}
    \resizebox{0.7\linewidth}{!}{
    \begin{tabular}{lcccc}
    \toprule
    \multirow{2}{*}{\textbf{Method}}
    & \multicolumn{4}{c}{\textbf{Bitwidth}} \\
    \cmidrule(lr){2-5}
    & \textbf{1-bit} & \textbf{2-bit} & \textbf{3-bit} & \textbf{4-bit} \\
    \midrule
    \tqmse                 & \textbf{0.0219} & \textbf{0.0220} & \textbf{0.0218} & \textbf{0.0222} \\
    \rbqbsm                      & 0.0247 & 0.0392 & 0.0394 & 0.0402 \\
    \edbsm                        & 0.0220 & \textbf{0.0222} & \textbf{0.0221} & \textbf{0.0223} \\
    \midrule
    \bqmse $(p=2)$         & 0.0220 & 0.0242 & 0.0345 & 0.0778 \\
    \bqmse $(p=3)$         & \textbf{0.0218} & 0.0301 & 0.0961 & 0.6531 \\
    \bqmseapprox $(p=3)$  & 0.0236 & 0.0237 & 0.0237 & 0.0265 \\
    \midrule
    \bqmse $(p=3)$          & 0.0249 & 0.0284 & 0.0294 & 0.0350 \\
    \bqmseapprox $(p=3)$   & 0.0273 & 0.0273 & 0.0277 & 0.0340 \\
    \bottomrule
    \end{tabular}
    }
\end{table}

For \bq{}, the exact assignment cost grows rapidly with both block size and bit-width, especially for $p=3$, where the runtime reaches $0.6573$ seconds at $4$ bits. 
In contrast, the approximate version removes this dependence on codebook size and remains nearly constant across bit-widths (e.g., $0.0251$--$0.0270$ seconds for \bqubapprox{}). 
A similar trend holds for MSE and UR variants in Table~\ref{tab:runtime_non_ip}.

Overall, the LUT-based approximation reduces the complexity from full codebook search to a small candidate search, bringing \bq{} to a runtime comparable with the fastest baselines while preserving its accuracy advantages.

\subsection{Computational Resources}
Except for the KV-cache quantization experiments, all experiments are conducted on a GPU
server with eight NVIDIA GeForce RTX 3090 GPUs, each with 24 GiB of VRAM, together with a dual-socket CPU server containing two Intel Xeon Gold 6226R processors, for a total of 32 cores and 32 threads. 
The KV-cache quantization experiments are conducted on a separate
server with four NVIDIA H100 SXM5 GPUs, each with 80 GiB of VRAM, and two Intel Xeon
Platinum 8592+ processors, totaling 128 cores and 128 threads.

\end{document}